\documentclass{article}

\usepackage{iclr2026_conference,times}

\usepackage{hyperref}
\usepackage{url}
\usepackage{booktabs}
\usepackage{amsfonts}
\usepackage{xcolor}

\usepackage{graphicx}
\usepackage{amsmath}
\usepackage{amssymb}
\usepackage{amsthm}
\usepackage{algorithm}
\usepackage{algorithmic}
\usepackage{arydshln}
\usepackage{colortbl}
\usepackage{subcaption}
\usepackage{float}
\definecolor{ourbg}{RGB}{232,243,255}
\definecolor{cperfect}{HTML}{1DB954}
\definecolor{cinform}{HTML}{2E8BF0}
\definecolor{cmislead}{HTML}{F04848}
\definecolor{gold}{HTML}{FFD700}
\definecolor{silver}{HTML}{90B0C8}
\definecolor{bronze}{HTML}{E08040}
\usepackage{tcolorbox}
\tcbuselibrary{listings,skins}
\usepackage{listings}
\newtcolorbox{promptbox}[1][]{
  colback=black!3,
  colframe=black!50,
  fonttitle=\bfseries\small,
  title={#1},
  boxrule=0.5pt,
  arc=2pt,
  left=4pt, right=4pt, top=2pt, bottom=2pt,
  before skip=4pt, after skip=4pt,
}
\lstdefinestyle{prompt}{
  basicstyle=\footnotesize\ttfamily\raggedright,
  breaklines=true,
  breakatwhitespace=false,
  columns=flexible,
  keepspaces=true,
  escapeinside={<@}{@>},
  aboveskip=0pt, belowskip=0pt,
  xleftmargin=0pt, xrightmargin=0pt,
  breakindent=0pt,
  postbreak={},
}

\newtheorem{theorem}{Theorem}
\newtheorem{lemma}[theorem]{Lemma}
\newtheorem{proposition}[theorem]{Proposition}
\newtheorem{corollary}[theorem]{Corollary}
\newtheorem{definition}[theorem]{Definition}
\newtheorem{assumption}[theorem]{Assumption}
\newtheorem*{theorem*}{Theorem}
\newtheorem*{proposition*}{Proposition}
\newtheorem*{corollary*}{Corollary}
\theoremstyle{remark}

\newcommand{\R}{\mathbb{R}}
\newcommand{\E}{\mathbb{E}}

\newcommand{\passk}{\text{Pass@}k}
\newcommand{\looauc}{\mathrm{LOO\text{-}AUC}}
\DeclareMathOperator{\AUC}{AUC}
\DeclareMathOperator{\rank}{rank}

\newcommand{\up}[1]{{\textcolor{teal}{\textsubscript{+#1}}}}
\newcommand{\dn}[1]{{\textcolor{red!70!black}{\textsubscript{-#1}}}}
\newcommand{\gc}[1]{\textcolor{black!50}{#1}}
 
\title{ACES: Who Tests the Tests? Leave-One-Out\\AUC Consistency for Code Generation}

\author{%
\textbf{Hui Sun$^{1,2,*}$ \quad Yun-Ji Zhang$^{1,2,*}$ \quad
Zheng Xie$^1$ \quad Ren-Biao Liu$^{1,2}$}\\
\textbf{Yali Du$^{1,2}$ \quad Xin-Ye Li$^{1,2}$ \quad Ming Li$^{1,2,\dagger}$}\\
$^1$National Key Laboratory for Novel Software Technology, Nanjing University, China\\
$^2$School of Artificial Intelligence, Nanjing University, China\\
{\small\texttt{\{sunh, zhangyunji, xiez, liurb, duyl, lixy, lim\}@lamda.nju.edu.cn}}\\
$^*$Equal contribution. \quad $^\dagger$Corresponding author.
}

\iclrfinalcopy

\hypersetup{
  pdftitle={ACES: Who Tests the Tests? Leave-One-Out AUC Consistency for Code Generation},
  pdfauthor={Hui Sun, Yun-Ji Zhang, Zheng Xie, Ren-Biao Liu, Yali Du, Xin-Ye Li, Ming Li},
}

\begin{document}

\maketitle
\pagestyle{plain}
\thispagestyle{plain}

\begin{abstract}
Selecting LLM-generated code candidates using LLM-generated tests is challenging because the tests themselves may be incorrect.
Existing methods treat all tests equally or use ad-hoc heuristics to filter unreliable tests.
Yet determining test correctness requires knowing which codes are correct, creating a \emph{circular dependency}.
Our key insight is that we need not determine test correctness at all: \emph{test votes should rank, not merely count}.
What matters is not how many codes pass a test, but whether it can \emph{distinguish} correct from incorrect code.
We break the circular dependency via leave-one-out evaluation: hold out one test, rank codes by their aggregate scores on all remaining tests, and measure whether the held-out test's pass/fail pattern agrees with this ranking.
We formalize this agreement as the leave-one-out AUC~(LOO-AUC) and prove that the expected LOO-AUC is proportional to each test's ability to separate correct code from incorrect code.
Building on this, we propose \textbf{ACES}~(\textbf{A}UC \textbf{C}onsist\textbf{E}ncy \textbf{S}coring) with two complementary variants: ACES-C provides closed-form weights that provably approximate the oracle in expectation under a mild assumption on average test quality; ACES-O drops this assumption and iteratively optimizes a differentiable LOO-AUC objective.
Both operate solely on the binary pass matrix with negligible overhead, and achieve state-of-the-art Pass@$k$ on multiple code generation benchmarks.
\end{abstract}
\section{Introduction}
\label{sec:intro}

Large language models~(LLMs) have demonstrated strong capabilities in code generation~\citep{chen2021evaluating, guo2024deepseek, pmlr-v267-liu25ah, cast_tse}, yet individual generations are not always correct.
A promising strategy is to scale test-time computation~\citep{snell2025scaling, brown2024large, wu2024inference}: generate many candidate solutions together with test cases, and select the best candidates based on test execution~\citep{li2022alphacode, chen2022codet}.
The central challenge is that \emph{neither the code nor the tests are guaranteed to be correct}: we need reliable tests to judge code quality, and reliable code to judge test quality, but have neither.

Existing methods either treat all tests equally via majority voting or rely on heuristics: CodeT~\citep{chen2022codet} scores consensus sets by size without weighting individual tests, while MBR-exec (Minimum Bayes Risk)~\citep{shi2022natural} and SRank~\citep{to2024functional} require pairwise output comparison beyond the binary pass matrix.
More heavyweight approaches co-evolve code and tests via reinforcement learning~\citep{wang2025cure} or evolutionary search~\citep{11098743}, requiring substantially more computation.
Yet the \emph{circular dependency} between code and test quality remains: to our knowledge, no existing method offers formal guarantees for identifying reliable tests without knowing which codes are correct.

What should such a criterion measure? We observe that code selection is fundamentally a ranking problem: we must order candidates so that correct solutions appear near the top.
Our key insight is that, for this ranking task, a test's value lies not in its correctness but in its ability to \emph{distinguish} correct code from incorrect code: a trivially correct test that all codes pass offers no ranking signal, while a demanding test that separates candidates is valuable even if imperfect.
\emph{Test votes should rank, not merely count.}
Under uniform counting, however, this distinction is lost: easy tests that most codes pass dilute the ranking signal, while tests that favor incorrect codes actively corrupt it.

How can we measure a test's ability to distinguish codes without knowing which codes are correct?
We break the circular dependency via \emph{leave-one-out evaluation among tests themselves}.
Hold out a single test from the generated test set; the remaining tests, aggregated via any reasonable weighting, induce a ranking of the code candidates.
If codes ranked highly by the other tests also tend to pass the held-out test, then the test is informative.
If it \emph{contradicts} the ranking, it is misleading.
If it is uncorrelated, it is uninformative.
This evaluation requires no knowledge of code correctness; it exploits only the internal structure of the pass matrix.

We formalize this as the leave-one-out AUC~(LOO-AUC): the area under the ROC curve with the remaining tests' aggregate scores as predictions and each held-out test's pass/fail column as the label.
A test contributes to ranking to the extent that correct codes are more likely to pass it than incorrect codes; we call this difference the test's \emph{discriminative power}.
This quantity is latent, since code correctness is unknown.
Yet LOO-AUC requires no knowledge of code correctness.
We prove that each test's expected LOO-AUC is proportional to its discriminative power, with a coefficient that depends on the test's pass-rate variance and the ranking quality of the remaining tests~(Theorem~\ref{lem:loo-identity}).
This provides the theoretical basis for principled non-uniform test weighting.

Building on this, we propose \emph{ACES}~(AUC Consistency Scoring) with two complementary variants.
\emph{ACES-C} applies the pass-rate correction in closed form under uniform weighting, provably approximating the oracle-optimal test weights in expectation under a mild assumption on average test quality~(Theorem~\ref{thm:main}).
\emph{ACES-O} instead iteratively optimizes the test weights through a differentiable LOO-AUC objective, without requiring this assumption.
Both operate solely on the binary pass matrix with negligible overhead.
The two are complementary: the average-quality assumption holds for the majority of tasks, and in this regime ACES-C's one-shot correction is near-optimal; for more challenging tasks where it fails, ACES-O's iterative optimization remains effective.

Our main contributions are as follows:
\begin{itemize}
  \item \textbf{Theoretical foundation.} We introduce the LOO-AUC identity~(Theorem~\ref{lem:loo-identity}), linking each test's observable consistency with the ranking to its latent discriminative power. To our knowledge, this is the first provable criterion for distinguishing informative from misleading tests using only the binary pass matrix.
  \item \textbf{Algorithms.} Building on this identity, we propose two lightweight algorithms. ACES-C provides closed-form weights provably approximating the oracle-optimal weights in expectation~(Theorem~\ref{thm:main}); ACES-O iteratively optimizes test weights via a differentiable LOO-AUC objective without requiring the average-quality assumption.
  \item \textbf{Empirical results.} We show that ACES advances the state-of-the-art in Pass@$k$ on multiple benchmarks, with ACES-O leading as a standalone method and ACES-C excelling as a plug-and-play execution-only scorer when combined with other pre-filtering mechanisms.
\end{itemize}
\section{Theoretical Foundations}
\label{sec:prelim}

We formalize code ranking as a weighted voting problem over the pass matrix and develop the theoretical tools that motivate ACES.

\subsection{Problem Setup}
\label{sec:setup}

Given a programming problem, an LLM generates $n$ candidate solutions $C = \{c_1, \ldots, c_n\}$ and $m$ test cases $T = \{t_1, \ldots, t_m\}$.
Executing every candidate on every test yields the \emph{pass matrix}
\begin{equation}\label{eq:pass-matrix}
  B \in \{0,1\}^{n \times m}, \qquad B_{ij} = \mathbf{1}[c_i \text{ passes } t_j].
\end{equation}
Each code has an unknown correctness label $y_i \in \{0,1\}$; we write $n^+ = |\{i : y_i = 1\}|$ and $n^- = n - n^+$.
Test quality is also unknown: some tests may have incorrect expected outputs.
Since test reliability varies, we parameterize code ranking by a weight vector over tests:
$w \in \Delta^m = \{w \in \R^m_+ : \sum_j w_j = 1\}$ induces scores $s_i(w) = \sum_{j=1}^m w_j B_{ij}$.
Codes are ranked by $s_i(w)$ in descending order~(ties broken uniformly at random).
We evaluate with $\passk$, the probability that at least one correct code appears in the top~$k$:
\begin{equation}\label{eq:passk}
  \passk(w) = P\!\left(\min_{i:\, y_i = 1} \rank(c_i) \leq k\right).
\end{equation}
The simplest baseline, \emph{majority voting}, sets $w_{\mathrm{unif}} = (1/m, \ldots, 1/m)$, ranking codes by total passes.

\textbf{Pairwise analysis.}
To understand how $w$ affects ranking, consider a correct code $c^+$ and an incorrect code $c^-$.
Each test $t_j$ casts a \emph{vote} $h_j = B_{c^+,j} - B_{c^-,j} \in \{+1, 0, -1\}$ on their ordering~(Table~\ref{tab:noise}), and the score difference $s_{c^+}(w) - s_{c^-}(w) = \sum_j w_j h_j$ is the weighted sum of these votes: code ranking is a \emph{weighted voting problem}.
Averaging over all such pairs, the probability that $w$ places the correct code above the incorrect one defines the \emph{true AUC}:
\begin{equation}\label{eq:auc-def}
  A(w) \;=\; P(s_{c^+}(w) > s_{c^-}(w)) \;+\; \frac{1}{2}\,P(s_{c^+}(w) = s_{c^-}(w)).
\end{equation}
Since $A(w)$ depends on unknown code labels, it cannot be computed directly and appears only in the analysis.
Under uniform weighting, misleading tests contribute as much as informative ones, actively lowering $A(w)$ and corrupting the ranking.
Constant columns of $B$~(all codes agree) yield only uninformative votes and are removed in preprocessing; we assume this throughout.

\subsection{Discriminative Power and Pass@k Bound}
\label{sec:framework}

Codes and tests are typically sampled independently from the LLM.
Conditioned on its correctness label $y_i$, each code's test outcomes are identically distributed and independent across tests.
This conditional independence is standard when codes and tests are sampled independently from the LLM, and is shared by classical item response theory.

\begin{definition}[Discriminative Power]\label{def:model}
  For each test $t_j$, define the \emph{class-conditional pass rates} and \emph{discriminative power}:
  \begin{equation*}
    \alpha_j = P(B_{ij} = 1 \mid y_i = 1), \qquad \beta_j = P(B_{ij} = 1 \mid y_i = 0), \qquad \delta_j = \alpha_j - \beta_j.
  \end{equation*}
  A test is \emph{informative}~($\delta_j > 0$), \emph{uninformative}~($\delta_j = 0$), or \emph{misleading}~($\delta_j < 0$) according to whether correct codes are more, equally, or less likely to pass it than incorrect ones.
\end{definition}

\begin{figure}[t]
\begin{minipage}[t]{0.52\textwidth}
  \vspace{0pt}
  \captionof{table}{Vote types for a test $t_j$ on a correct-incorrect pair $(c^+, c^-)$.
  Each test casts a vote $h_j = B_{c^+,j} - B_{c^-,j}$ on their ordering;
  the score difference $s_{c^+}\!-\!s_{c^-} = \sum_j w_j h_j$ is the weighted aggregate.}
  \label{tab:noise}
  \vspace{0.3em}
  \centering
  \footnotesize
  \setlength{\tabcolsep}{3pt}
  \begin{tabular}{@{}lccc@{}}
    \toprule
    \textbf{Type} & \textbf{Condition} & $h_j$ & \textbf{Effect} \\
    \midrule
    Informative & $c^+$ pass, $c^-$ fail & $+1$ & Raises $s_{c^+}\!-\!s_{c^-}$ \\
    Uninformative & Both same & $0$ & No effect \\
    Misleading & $c^-$ pass, $c^+$ fail & $-1$ & Lowers $s_{c^+}\!-\!s_{c^-}$ \\
    \bottomrule
  \end{tabular}
\end{minipage}%
\hfill
\begin{minipage}[t]{0.46\textwidth}
  \vspace{0pt}
  \centering
  \includegraphics[width=0.96\linewidth]{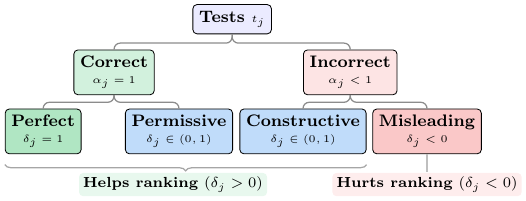}
  \caption{\textbf{For ranking, only $\delta_j$ matters.}
  Taxonomy of tests by correctness ($\alpha_j$) and discriminative power ($\delta_j$).}
  \label{fig:taxonomy}
\end{minipage}
\end{figure}
 
Figure~\ref{fig:taxonomy} refines this classification by test correctness: correct tests~($\alpha_j = 1$) are always informative, while incorrect tests~($\alpha_j < 1$) may be constructive~($\delta_j > 0$) or misleading~($\delta_j < 0$).
Prior correctness-based filters discard constructive tests together with misleading ones, losing valuable ranking signal; for ranking, only $\delta_j$ matters.
Equivalently, $\delta_j = \E[h_j]$ is the expected pairwise vote.
$\passk$ is controlled by a signal-to-noise ratio determined by $w$ and $\{\delta_j\}$:

\begin{theorem}[Pass@k Bound]\label{thm:hoeffding}
  Define the \emph{mean signal} and \emph{signal-to-noise ratio} of weights $w$:
  \begin{equation*}
    M(w) \;:=\; \sum_j w_j\delta_j, \qquad R(w) \;:=\; \frac{M(w)^2}{\sum_j w_j^2}.
  \end{equation*}
  For any $w$ with $M(w) > 0$ and any $k \geq 1$:
  \begin{equation}\label{eq:passk-exp}
    \passk(w) \;\geq\; 1 - \frac{n^-}{k}\,\exp\!\left(-\frac{R(w)}{2}\right),
  \end{equation}
  where $n^- = n - n^+$ is the number of incorrect codes.
  Over non-negative weights, $R(w)$ is maximized by $w^*_j \propto \max(0,\;\delta_j)$.
\end{theorem}

The proof is in Appendix~\ref{app:proof-lemma-passk}.
The bound improves exponentially with $R(w)$, but the oracle-optimal weights~$w^*$ require knowing the unknown discriminative powers~$\delta_j$.

\textbf{Majority voting baseline.}
Writing $\bar{\delta} := (1/m)\sum_{j=1}^m \delta_j$, substituting uniform weights gives $R(w_{\mathrm{unif}}) = m\bar{\delta}^2$ and hence:
\begin{equation}\label{eq:mv-bound}
  \passk(w_{\mathrm{unif}}) \;\geq\; 1 - \frac{n^-}{k}\,\exp\!\left(-\frac{m\bar{\delta}^2}{2}\right).
\end{equation}
Non-uniform weights can improve this bound by increasing $R(w)$ beyond $m\bar{\delta}^2$, but the optimal weights~$w^*$ are unknown in practice.

\subsection{Estimating Test Quality for Ranking: LOO-AUC Identity}
\label{sec:loo-theory}

The discriminative powers $\delta_j$ are unknown, but the pass matrix $B$ provides $m$ surrogate label vectors, one per test column.
An informative test's column correlates with the unknown $y$, while a misleading test's column anti-correlates.
To distinguish the two, one can hold out test $t_j$, rank codes by their aggregate scores on the remaining tests, and measure whether $t_j$'s pass/fail pattern agrees with that ranking.
Formally, the \emph{leave-one-out scores} and \emph{LOO-AUC} are:
\begin{equation}\label{eq:loo-score}
  S^{(-j)}_i = \sum_{j' \neq j} w_{j'}\, B_{ij'}, \qquad
  \looauc_j(w) = \AUC\!\left(S^{(-j)},\; B_{:,j}\right),
\end{equation}
where $S^{(-j)}$ serves as scores and $B_{:,j}$ as binary labels.
Both are computed from the pass matrix $B$ alone; \textbf{no knowledge of $y$ is required}.
$\looauc_j$ measures the \emph{consistency} between test~$t_j$ and the remaining tests: the probability that, among codes distinguished by $t_j$, those passing it are ranked higher by the remaining tests.
A high value ($\looauc_j > 1/2$) indicates agreement with the consensus ranking; a low value ($\looauc_j < 1/2$) indicates disagreement.
The following identity formalizes how this consistency measure relates to discriminative power~$\delta_j$:
\begin{theorem}[LOO-AUC Identity]\label{lem:loo-identity}
  Let $A^{(-j)}(w)$ denote the true AUC~(Eq.~\ref{eq:auc-def}) computed from all tests except~$t_j$.
  For any weights $w$:
  \begin{equation}\label{eq:loo-proportional}
    \E[\looauc_j(w)] - \frac{1}{2} \;=\; c_j(w) \cdot \delta_j,
    \qquad c_j(w) \;:=\; \frac{\pi(1-\pi)}{p_j(1-p_j)}\left(A^{(-j)}(w) - \frac{1}{2}\right),
  \end{equation}
  where $\pi = n^+/n$ is the correct-code fraction and $p_j = P(B_{ij} = 1)$ the marginal pass rate of~$t_j$.
\end{theorem}

The proof is in Appendix~\ref{app:proof-loo-identity}.
The coefficient $c_j(w)$ has two test-dependent factors: the inverse pass-rate variance $1/[p_j(1-p_j)]$ and the leave-one-out ranking quality $A^{(-j)}(w) - 1/2$.
In Section~\ref{sec:method}, we exploit this decomposition to construct test weights that target the discriminative powers~$\delta_j$.
\section{ACES: AUC Consistency Scoring}
\label{sec:method}

We use the LOO-AUC Identity~(Theorem~\ref{lem:loo-identity}) to construct non-uniform test weights that improve Pass@$k$ through two complementary approaches: ACES-C applies a closed-form pass-rate correction, while ACES-O optimizes weights via differentiable LOO-AUC.

\subsection{ACES-C: Closed-Form Weighting}
\label{sec:aces}

The LOO-AUC Identity~(Theorem~\ref{lem:loo-identity}) gives $\E[\looauc_j(w)] - 1/2 = c_j(w) \cdot \delta_j$.
When $c_j(w) > 0$, LOO-AUC excess shares the same sign as~$\delta_j$, so tests scoring above $1/2$ are identified as informative and those below as misleading.
This holds whenever $A^{(-j)}(w) > 1/2$, i.e., the remaining tests rank correct codes higher, the natural regime since the complement would mean the aggregate ranking is no better than random.
To guarantee this for all~$j$ under uniform weights, we introduce:

\begin{assumption}\label{asm:main}
  The average discriminative power is positive, and the test pool is large enough that leaving out any single test has negligible effect on the aggregate ranking. Quantitatively:
  \begin{equation*}
    \bar{\delta} \;:=\; \frac{1}{m}\sum_{j=1}^m \delta_j \;>\; 2\sqrt{\frac{\ln 2}{m}}.
  \end{equation*}
\end{assumption}

This is our only assumption beyond the model in Section~\ref{sec:framework}; the threshold is $O(1/\sqrt{m})$, and unlike the classical \emph{weak learner} condition~\citep{freund2003efficient}, which requires \emph{every} ranker to be better than random, ours requires this only \emph{on average}, permitting arbitrarily many misleading tests.

\begin{proposition}[Structure of $c_j(w_{\mathrm{unif}})$]\label{thm:loo-quality}
  Under Assumption~\ref{asm:main}, $A^{(-j)}(w_{\mathrm{unif}}) > 1/2$ for all~$j$, so $c_j(w_{\mathrm{unif}}) > 0$.
  Moreover, $A^{(-j)}(w_{\mathrm{unif}})$ is approximately constant across~$j$.
\end{proposition}

The proof and quantitative bound are in Appendix~\ref{app:proof-loo-quality}.
Under uniform weights, $A^{(-j)}(w_{\mathrm{unif}}) - 1/2$ is approximately constant across $j$~(Proposition~\ref{thm:loo-quality}).
Combined with the constant $\pi(1-\pi)$, the only significant test-dependent factor in $c_j(w_{\mathrm{unif}})$ is $1/[p_j(1-p_j)]$.
Multiplying the LOO-AUC excess by $p_j(1-p_j)$ removes this distortion:
\begin{equation}\label{eq:aces-weight}
  w_j \;=\; \max\!\left(0,\; \looauc_j(w_{\mathrm{unif}}) - \frac{1}{2}\right) \cdot p_j(1-p_j),
\end{equation}
where $p_j = \frac{1}{n}\sum_{i=1}^n B_{ij}$ is the empirical pass rate of test~$j$.
The $\max(0,\cdot)$ ensures non-negative weights, assigning zero weight to tests with $\looauc_j \leq 1/2$ and filtering out those identified as misleading.
Since only relative magnitudes affect the induced ranking, normalization is unnecessary.
The following theorem shows that the pass-rate-corrected LOO-AUC excess recovers the discriminative power~$\delta_j$ exactly in expectation.

\begin{theorem}[ACES-C Recovers Discriminative Power]\label{thm:main}
  Under Assumption~\ref{asm:main}, define the \emph{quality score} $q_j := (\looauc_j(w_{\mathrm{unif}}) - \tfrac{1}{2})\,p_j(1{-}p_j)$.  Then:
  \begin{equation*}
    \E[q_j] \;=\; \pi(1-\pi)\!\left(A^{(-j)}(w_{\mathrm{unif}}) - \frac{1}{2}\right) \cdot \delta_j.
  \end{equation*}
  In particular, $\E[q_j] > 0$ if and only if $\delta_j > 0$: the expected quality score's sign identifies informative versus misleading tests.
  Since $A^{(-j)}(w_{\mathrm{unif}})$ is nearly constant in~$j$~(Proposition~\ref{thm:loo-quality}), $\E[q_j] \propto \delta_j$.
\end{theorem}

\begin{proof}
  By Theorem~\ref{lem:loo-identity}, $\E[\looauc_j] - 1/2 = c_j \delta_j$ with $c_j = \frac{\pi(1-\pi)}{p_j(1-p_j)}(A^{(-j)} - 1/2) > 0$ under Assumption~\ref{asm:main}.
  Multiplying by $p_j(1{-}p_j)$ cancels the $1/[p_j(1{-}p_j)]$ factor in~$c_j$, yielding $\E[q_j] = \pi(1{-}\pi)(A^{(-j)} - 1/2)\,\delta_j$.
  Positivity when $\delta_j > 0$ follows from $A^{(-j)} > 1/2$~(Proposition~\ref{thm:loo-quality}).
\end{proof}

The ACES-C weight~(Eq.~\ref{eq:aces-weight}) is $w_j = \max(0,\, q_j)$: by the sign property above, the clipping assigns positive weight precisely to tests that are informative in expectation~($\delta_j > 0$), targeting the oracle $w^*_j \propto \max(0,\, \delta_j)$.
Quantitatively, the ACES-C weights achieve near-oracle signal-to-noise ratio: $R \geq \rho^2 R^*$ where $\rho \to 1$ as $m$ grows~(Corollary~\ref{cor:improvement} in Appendix~\ref{app:improvement}).

\subsection{ACES-O: Optimized Weighting}
\label{sec:aces-o}

ACES-C is an efficient closed-form method with provable guarantees under Assumption~\ref{asm:main}.
When the assumption may fail, ACES-O jointly optimizes the weights and their induced ranking:
\begin{equation}\label{eq:objective}
  \max_{w \in \Delta^m} \;\; J(w) \;=\; \sum_{j=1}^{m} w_j\,\bigl(\looauc_j(w) - \frac{1}{2}\bigr),
\end{equation}
where $\looauc_j(w)$ depends on $w$ through the leave-one-out scores $S^{(-j)}(w)$~(Eq.~\ref{eq:loo-score}).
The simplex constraint leaves the induced ranking unchanged but prevents unbounded weights and ill-defined optimization.
ACES-C evaluates LOO-AUC once under uniform weights, whereas ACES-O reevaluates it under the current weights~$w$.
In ACES-O, tests with LOO-AUC below $1/2$ are softly down-weighted, not permanently excluded; as the leave-one-out ranking improves, their values can rise above $1/2$, recovering initially excluded informative tests.

\textbf{Theoretical motivation.}
The LOO-AUC identity~(Theorem~\ref{lem:loo-identity}) gives
\begin{equation}\label{eq:expected-obj}
  \E[J(w)] \;=\; \sum_{j=1}^{m} w_j\, c_j(w)\, \delta_j.
\end{equation}
When $c_j(w) > 0$, each term has the same sign as~$\delta_j$: informative tests increase $\E[J(w)]$ while misleading tests decrease it.
Maximizing $J(w)$ thus drives the weights toward informative tests and away from misleading ones, without requiring knowledge of~$\delta_j$.
Assumption~\ref{asm:main} guarantees $c_j(w_{\mathrm{unif}}) > 0$ for all~$j$~(Proposition~\ref{thm:loo-quality}), providing a valid starting point.
As the optimization updates~$w$, the leave-one-out ranking improves, raising~$c_j(w)$ for previously misidentified tests and creating a positive feedback loop that can succeed even when the assumption is only weakly satisfied.
Since the AUC is non-differentiable, we optimize a logistic surrogate via gradient ascent; pseudocode, pre-filtering, and hyperparameters are in Appendix~\ref{app:algorithms}.
Appendix~\ref{app:example} illustrates this complementarity with two constructed examples: ACES-C suffices when Assumption~\ref{asm:main} is well met, whereas ACES-O recovers initially missed informative tests when it is only weakly met.
\section{Experiments}
\label{sec:experiments}

We evaluate ACES-C and ACES-O on three code generation benchmarks, focusing on: (1)~comparison with existing reranking methods, (2)~empirical validation of Assumption~\ref{asm:main}, and (3)~analysis of robustness and test quality detection under LOO-AUC-based weighting.

\subsection{Setup}
\label{sec:exp-setup}

\textbf{Benchmarks and generation.}
We evaluate on HumanEval~\citep{chen2021evaluating}~(164 problems), HumanEval$^+$~\citep{liu2023your}~(164 problems, stricter tests), and MBPP~\citep{austin2021program}~(427 problems).
We use candidate solutions and tests from \citet{huang2024enhancing}, generated by GPT-3.5-Turbo, with approximately 200 candidates and 500 tests per problem.
We report Pass@$k$ for $k \in \{1, 2, 5\}$.
For reranking methods, Pass@$k$ follows Eq.~\ref{eq:passk}; for direct inference baselines, Pass@$k$ is the unbiased estimator of \citet{chen2021evaluating}, averaging over random $k$-subsets of candidates.

\textbf{Baselines.}
We compare with post-hoc methods using code$\times$test execution:
Majority Voting~($w_{\mathrm{unif}}$),
CodeT~\citep{chen2022codet}~(consensus-set scoring on the pass matrix),
and MBR-exec~\citep{shi2022natural}~(pairwise output comparison beyond the pass matrix);
and methods requiring additional information:
SC+Spec~\citep{huang2024enhancing}~(execution + specification consistency voting),
Self-collaborate~\citep{dong2024self}~(multi-agent LLM calls),
MPSC~\citep{huang2024enhancing}~(multi-perspective consistency voting; Uniform variant, using only inter-consistency),
and $\mathcal{DS}^3$~\citep{Liu_Xue_Ma_Sun_Li_Li_2026}~(static analysis).
We also include direct inference from GPT-3.5-Turbo, GPT-4~\citep{achiam2023gpt}, DeepSeek-Coder~\citep{guo2024deepseek}, WizardCoder~\citep{luo2024wizardcoder}, and CodeLlama~\citep{roziere2023code}.
All post-hoc methods share GPT-3.5-Turbo candidates and tests.

Implementation details of ACES-C and ACES-O are in Appendix~\ref{app:algorithms}; results with additional generation models and benchmarks are in Appendix~\ref{app:additional-experiments}.

\subsection{Main Results}
\label{sec:main-results}

\begin{table}[t]
  \centering
  \caption{Pass@$k$ (\%) on HumanEval, HumanEval$^+$, and MBPP. Post-hoc methods rerank GPT-3.5-Turbo candidates ($n{\approx}200$, $m{\approx}500$). Subscripts denote the change from GPT-3.5-Turbo direct inference. \textbf{Bold}/\underline{underline}: best/second-best within each group. \colorbox{ourbg}{Shaded}: ours.}
  \label{tab:main-results}
  \vspace{0.5em}
  \small
  \setlength{\tabcolsep}{2.2pt}
  \begin{tabular}{@{}l lll lll lll@{}}
    \toprule
    & \multicolumn{3}{c}{\textbf{HumanEval}} & \multicolumn{3}{c}{\textbf{HumanEval$^+$}} & \multicolumn{3}{c}{\textbf{MBPP}} \\
    \cmidrule(lr){2-4} \cmidrule(lr){5-7} \cmidrule(lr){8-10}
    \textbf{Method} & Pass@1 & Pass@2 & Pass@5 & Pass@1 & Pass@2 & Pass@5 & Pass@1 & Pass@2 & Pass@5 \\

    \midrule
    \rowcolor{black!5}
    \multicolumn{10}{@{}l}{\textit{Direct inference}} \\
    GPT-3.5-Turbo        & 68.38 & 76.24 & 83.15 & 58.75 & 66.58 & 73.96 & 66.80 & 72.34 & 76.60 \\
    GPT-4                & 81.48 & 86.31 & 90.46 & 70.52 & 75.48 & 79.54 & 71.26 & 74.27 & 76.99 \\
    DeepSeek-Coder       & 79.30 & --    & --    & --    & --    & --    & 70.00 & --    & --    \\
    WizardCoder          & 73.20 & --    & --    & --    & --    & --    & 61.20 & --    & --    \\
    CodeLlama            & 62.20 & --    & --    & --    & --    & --    & 62.20 & --    & --    \\

    \midrule
    \rowcolor{black!5}
    \multicolumn{10}{@{}l}{\textit{Post-hoc reranking using only code $\times$ test execution}} \\
    MBR-exec             & 72.96\up{4.6} & 76.47\up{0.2} & 79.00\dn{4.2} & 62.12\up{3.4} & 67.08\up{0.5} & 71.38\dn{2.6} & 70.79\up{4.0} & \underline{73.14}\up{0.8} & \textbf{74.85}\dn{1.8} \\
    CodeT                & 78.05\up{9.7} & 78.05\up{1.8} & 78.30\dn{4.9} & 67.87\up{9.1} & 68.75\up{2.2} & 69.65\dn{4.3} & \underline{71.90}\up{5.1} & 71.95\dn{0.4} & 72.02\dn{4.6} \\
    Majority Voting      & 80.49\up{12.1} & \underline{82.93}\up{6.7} & \underline{83.54}\up{0.4} & 69.51\up{10.8} & \underline{73.17}\up{6.6} & \underline{76.83}\up{2.9} & 68.62\up{1.8} & 70.49\dn{1.9} & 72.83\dn{3.8} \\
    \rowcolor{ourbg}
    \textbf{ACES-C}      & \underline{82.93}\up{14.6} & \underline{82.93}\up{6.7} & 82.93\dn{0.2} & \underline{71.34}\up{12.6} & 71.95\up{5.4} & 74.39\up{0.4} & 71.19\up{4.4} & 71.43\dn{0.9} & 72.37\dn{4.2} \\
    \rowcolor{ourbg}
    \textbf{ACES-O}      & \textbf{84.15}\up{15.8} & \textbf{85.98}\up{9.7} & \textbf{86.59}\up{3.4} & \textbf{74.39}\up{15.6} & \textbf{75.61}\up{9.0} & \textbf{79.88}\up{5.9} & \textbf{72.37}\up{5.6} & \textbf{73.54}\up{1.2} & \underline{73.77}\dn{2.8} \\

    \midrule
    \rowcolor{black!5}
    \multicolumn{10}{@{}l}{\textit{Post-hoc reranking using additional information}} \\
    SC+Spec              & 73.86\up{5.5} & 73.93\dn{2.3} & 74.10\dn{9.1} & 63.50\up{4.8} & 64.70\dn{1.9} & 65.67\dn{8.3} & 71.70\up{4.9} & 71.73\dn{0.6} & 71.82\dn{4.8} \\
    Self-collaborate     & 74.40\up{6.0} & --    & --    & --    & --    & --    & 68.20\up{1.4} & --    & --    \\
    MPSC                 & 74.17\up{5.8} & 77.02\up{0.8} & 78.53\dn{4.6} & 65.05\up{6.3} & 69.76\up{3.2} & 71.72\dn{2.2} & 69.34\up{2.5} & 70.06\dn{2.3} & 71.85\dn{4.8} \\
    $\mathcal{DS}^3$     & 81.71\up{13.3} & 82.32\up{6.1} & 82.32\dn{0.8} & 72.56\up{13.8} & 73.78\up{7.2} & 75.00\up{1.0} & 75.88\up{9.1} & \textbf{77.28}\up{4.9} & 78.45\up{1.9} \\
    \rowcolor{ourbg}
    \textbf{ACES-C + $\mathcal{DS}^3$}  & \textbf{85.37}\up{17.0} & \textbf{85.98}\up{9.7} & \textbf{87.20}\up{4.1} & \textbf{77.44}\up{18.7} & \textbf{78.66}\up{12.1} & \textbf{81.10}\up{7.1} & \underline{76.11}\up{9.3} & \textbf{77.28}\up{4.9} & \textbf{78.69}\up{2.1} \\
    \rowcolor{ourbg}
    \textbf{ACES-O + $\mathcal{DS}^3$}  & \underline{83.54}\up{15.2} & \underline{83.54}\up{7.3} & \underline{85.98}\up{2.8} & \underline{75.00}\up{16.3} & \underline{76.22}\up{9.6} & \underline{79.88}\up{5.9} & \textbf{76.58}\up{9.8} & \textbf{77.28}\up{4.9} & \textbf{78.69}\up{2.1} \\
    \bottomrule
  \end{tabular}
\end{table}
 
Table~\ref{tab:main-results} summarizes the results. Using only the binary pass matrix, ACES achieves the best Pass@$k$ among all execution-based methods on all three benchmarks. When further combined with complementary static-analysis methods, it yields the best overall results across all benchmarks.

\textbf{Execution-only methods.}
Among execution-only methods, ACES-O achieves the best Pass@$k$ on all three benchmarks.
On HumanEval, ACES-O reaches 84.15\% Pass@1 versus 81.71\% for $\mathcal{DS}^3$, which leverages additional static-analysis signals beyond the pass matrix.
On HumanEval$^+$, whose stricter criteria increase the fraction of misleading tests among the same generated test suite, ACES-O reaches 74.39\% Pass@1, again outperforming $\mathcal{DS}^3$~(72.56\%) despite using only the pass matrix.
The gain over Majority Voting widens relative to HumanEval~(+4.88\% vs.\ +3.66\% Pass@1), consistent with our theory: when misleading tests are more prevalent, principled test weighting provides greater benefit~(Section~\ref{sec:analysis}).
On MBPP, ACES-O leads all execution-only methods~(72.37\% Pass@1) but trails $\mathcal{DS}^3$~(75.88\%), whose orthogonal static-analysis signals are complementary.
Overall, ACES surpasses all baselines on HumanEval and HumanEval$^+$ using only the pass matrix; on MBPP, only $\mathcal{DS}^3$~(static analysis) ranks higher.

\textbf{Combination with static analysis.}
$\mathcal{DS}^3$ operates in two stages: heuristic pre-filtering of candidates, followed by static-analysis scoring~(Pylint, AST similarity, cyclomatic complexity).
We reuse its first-stage filtering and combine the second-stage static-analysis scores with ACES scores via a weighted sum~(details in Appendix~\ref{app:algorithms}).
ACES-C + $\mathcal{DS}^3$ improves over $\mathcal{DS}^3$ alone on all three benchmarks~($+3.66$/$+4.88$/$+0.23$ Pass@1 on HumanEval/HumanEval$^+$/MBPP); ACES-O + $\mathcal{DS}^3$ similarly improves, with the largest gain on MBPP~($+0.70$), confirming that the two methods capture orthogonal signals.
Interestingly, ACES-C + $\mathcal{DS}^3$ outperforms ACES-O + $\mathcal{DS}^3$ on HumanEval and HumanEval$^+$~(e.g., 85.37\% vs.\ 83.54\% Pass@1 on HumanEval), because pre-filtering improves candidate quality and makes Assumption~\ref{asm:main} better satisfied, a regime where ACES-C's closed-form correction is already near-optimal~(as illustrated in Appendix~\ref{app:example}).
In practice, ACES-O extracts more signal for standalone reranking, while ACES-C's closed-form weights integrate more effectively with high-quality pre-filtering and require no optimization.

\textbf{Comparison of ACES variants.}
ACES-C and ACES-O are alternative weighting schemes rather than cumulative components. ACES-C improves Pass@1 over Majority Voting on all three benchmarks~(e.g., 82.93\% vs.\ 80.49\% on HumanEval), showing that closed-form LOO-AUC weighting captures meaningful test-quality signal; ACES-O outperforms ACES-C across all reported benchmark--$k$ settings.
Further analyses, including finer Pass@$k$ grids, selection vs.\ weighting, hyperparameter sensitivity, and computational cost, are in Appendix~\ref{app:additional-experiments}.

\subsection{Analysis}
\label{sec:analysis}

\begin{figure}[t]
  \centering
  \begin{subfigure}[t]{0.49\textwidth}
    \centering
    \includegraphics[width=\textwidth]{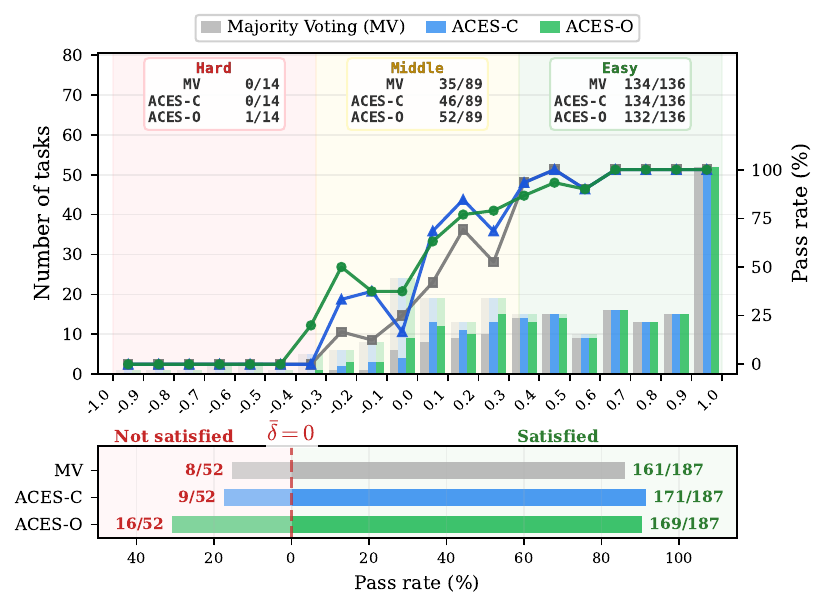}
    \caption{Assumption~\ref{asm:main} satisfaction on MBPP at Pass@1. Bottom: pass rate when $\bar{\delta}>0$ vs.\ $\bar{\delta}\leq 0$.}
    \label{fig:assumption-mbpp}
  \end{subfigure}%
  \hfill
  \begin{subfigure}[t]{0.5\textwidth}
    \centering
    \includegraphics[width=\textwidth]{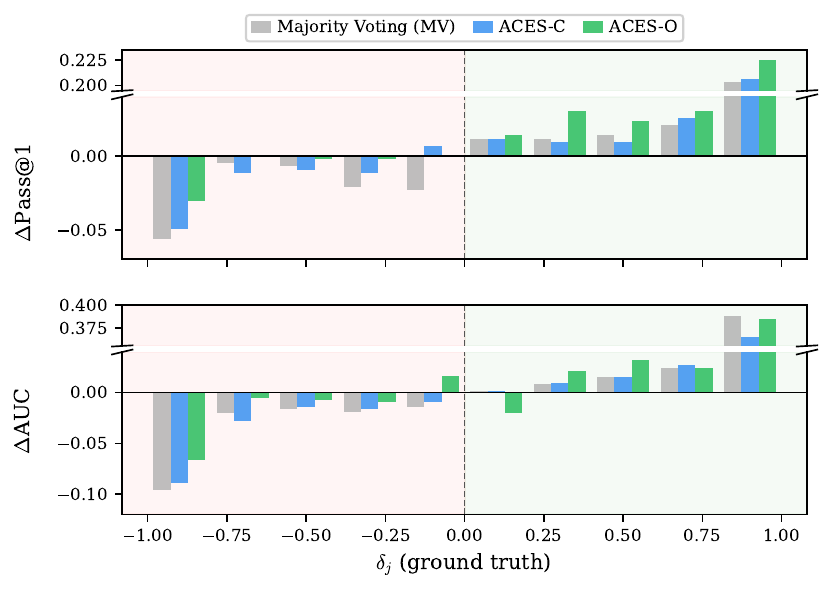}
    \caption{Impact of each $\delta_j$ bin on Pass@1~(top) and AUC~(bottom). Positive: helpful; negative: harmful.}
    \label{fig:binremoval}
  \end{subfigure}

  \caption{\textbf{Empirical analysis on MBPP at Pass@1.}
  \textbf{(a)}~Tasks binned by $\bar{\delta}$; bars show pass/fail counts, lines show pass rate; bottom panel compares pass rates by assumption status.
  \textbf{(b)}~Performance impact of each $\delta_j$ bin upon removing its tests; full results in Appendix~\ref{app:assumption}.}
  \label{fig:analysis}
\end{figure}

We restrict this analysis to non-trivial tasks~($0 < n^+ < n$; details in Table~\ref{tab:assumption-summary}) and non-constant tests~($0 < p_j < 1$), since neither trivial tasks nor constant test columns carry ranking signal.

\textbf{Assumption satisfaction.}
Figure~\ref{fig:assumption-mbpp} empirically validates Assumption~\ref{asm:main} on MBPP, which has the most non-trivial tasks and the lowest satisfaction rate among the three benchmarks.
We bin tasks by $\bar{\delta}$ and use the simplified threshold $\bar{\delta} > 0$ for visual clarity~(the formal threshold $2\sqrt{\ln 2/m}$ is small for $m \approx 500$; see Appendix~\ref{app:assumption} for the formal analysis).
The assumption is well satisfied on all benchmarks; detailed satisfaction rates are in Table~\ref{tab:assumption-summary}.
The performance gains concentrate in the Middle region, where informative and misleading tests coexist: ACES-O passes 52/89 tasks versus 46/89 for ACES-C and 35/89 for Majority Voting.
In the Easy region, all methods achieve near-perfect pass rates~($\geq 97\%$), leaving little room for improvement; in the Hard region, no method passes more than 1 of 14 tasks, suggesting the pass matrix alone is insufficient when average test quality is very low and complementary signals such as static analysis may be needed~(cf.\ Section~\ref{sec:main-results}).
The bottom panel shows that both ACES variants outperform Majority Voting whether or not the assumption holds; when it fails~($\bar{\delta} \leq 0$), ACES-O substantially outperforms ACES-C~(16/52 vs.\ 9/52), showing that iterative optimization is particularly valuable in this regime.

\textbf{Impact of test quality on performance.}
To understand \emph{why} ACES-O outperforms, we measure each $\delta_j$ bin's impact on performance as the drop incurred by removing its tests~(Figure~\ref{fig:binremoval}).
The key finding is an asymmetric dependence on test quality.
In the misleading-test region~($\delta_j < 0$), both ACES variants are more robust: the most misleading bin~($\delta_j \approx {-}0.9$) reduces Pass@1 by $0.056$ for Majority Voting but by only $0.049$ for ACES-C and $0.030$ for ACES-O~($46\%$ less sensitive).
This is consistent with the ACES weights having already down-weighted misleading tests, so that their presence or absence has little effect on the final ranking.
Conversely, informative tests~($\delta_j > 0$) contribute more to ACES-O than to Majority Voting~(e.g., $0.030$ vs.\ $0.012$ at $\delta_j \approx 0.3$), because the optimization has up-weighted these high-quality tests.
This asymmetry shows that LOO-AUC-based weighting effectively separates informative from misleading tests in the pass matrix.

\begin{figure}[t]
  \centering
  \includegraphics[width=\textwidth]{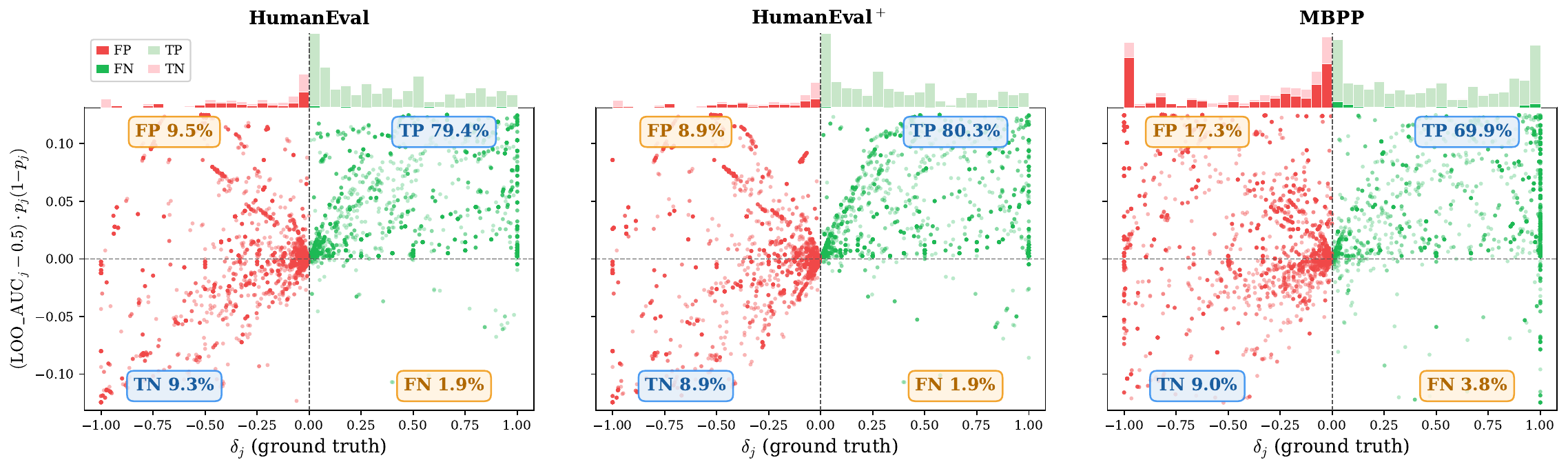}
  \caption{\textbf{Test quality detection.}
  ACES-C weight $(\looauc_j - \tfrac{1}{2})\,p_j(1{-}p_j)$ vs.\ ground-truth discriminative power~$\delta_j$ for all non-trivial tests across three benchmarks.
  Green/red points denote informative/misleading tests~($\delta_j > 0$/$\delta_j < 0$); top marginals give per-bin TP/FP/FN/TN~(saturated: errors; light: correct), and quadrants give TP/FP/FN/TN shares by weight sign~(summing to 100\%).}
  \label{fig:loo-scatter}
\end{figure}

\textbf{Test quality detection.}
Figure~\ref{fig:loo-scatter} compares ACES-C weights with ground-truth~$\delta_j$ across all benchmarks.
Theorem~\ref{lem:loo-identity} predicts that $\E[\looauc_j] - \tfrac{1}{2}$ and $\delta_j$ share a sign; accordingly, weight signs identify at least $94.8\%$ of informative tests.
The top marginals show FP/FN errors cluster near $\delta_j \approx 0$, where weak discriminative power limits their ranking impact.
MBPP exceeds HumanEval in FP fraction~($17.3\%$ vs.\ $9.5\%$) and misleading-test prevalence~(FP${+}$TN: $26.3\%$ vs.\ $18.8\%$); because errors cluster near $\delta_j=0$, this gap lies mainly among borderline tests.
\section{Related Work}
\label{sec:related}

\textbf{Code selection with generated tests.}
Post-hoc selection methods rerank candidates from code LLMs~\citep{lozhkov2024starcoder, hui2024qwen2} using execution-based signals.
CodeT~\citep{chen2022codet} groups candidates into consensus sets by their binary pass profiles and scores each set by the number of passed tests, operating entirely on the pass matrix.
Several methods exploit execution outputs beyond the pass matrix: MBR-exec~\citep{shi2022natural} applies minimum Bayes risk decoding~\citep{eikema2022sampling} via pairwise output comparison, SRank~\citep{to2024functional} clusters candidates by output equivalence, and ALGO~\citep{zhang2023algo} verifies against LLM-synthesized oracle programs.
S$^*$~\citep{li-etal-2025-test} uses LLM-guided tournament-style selection with adaptively generated inputs.
Others incorporate non-execution signals: Coder-Reviewer~\citep{zhang2023coder} uses generation log-probabilities, and LEVER~\citep{ni2023lever} trains a verifier on code and execution features.
MPSC~\citep{huang2024enhancing} adds specification-level consistency voting, and $\mathcal{DS}^3$~\citep{Liu_Xue_Ma_Sun_Li_Li_2026} augments execution with static analysis.
CURE~\citep{wang2025cure} and CoCoEvo~\citep{11098743} jointly improve code and test generators via reinforcement learning or evolutionary search, at the cost of model training; self-debugging~\citep{chen2023teaching} iteratively refines candidates using execution feedback rather than selecting among them.
On the test-quality side, ConVerTest~\citep{convertest2026} applies self-consistency filtering~\citep{wang2022self} before voting, and TestCase-Eval~\citep{yang-etal-2025-llms} evaluates LLM-generated test reliability; recent systematic studies~\citep{li2024doce, sun2024sifting} further examine how execution-based strategies compose.
Across these approaches, no formal criterion has been established for identifying test quality from the pass matrix alone; ACES fills this gap with a provable criterion for distinguishing informative from misleading tests using only binary execution outcomes.

\textbf{Ranking and noisy evaluation.}
Code selection via weighted voting is a bipartite ranking problem~\citep{clemenson2008ranking, agarwal2005generalization, liu2009learning}.
Pairwise surrogate consistency~\citep{gao2015auc} justifies our logistic loss.
Recent extensions address partial~\citep{wang2022optimizing}, weakly supervised~\citep{xie2023wsauc}, and multi-label~\citep{pmlr-v267-lukasik25a} AUC optimization; our setting fits naturally into the last framework, where each test acts as a noisy annotator~\citep{dawid1979maximum}.
Unlike RankBoost~\citep{freund2003efficient}, which requires every ranker to beat random, our Assumption~\ref{asm:main} requires this only \emph{on average}.
The noisy-comparisons literature provides minimax rates~\citep{shah2018simple} and noise tolerance results~\citep{haddad2022noise}, but does not address how to identify reliable comparators.
\citet{nguyen2024noisy} show that multiple annotators are necessary for such identification; our LOO-AUC mechanism addresses this without external supervision.
The verifier paradigm~\citep{cobbe2021training, lightman2023let} and LLM-as-judge evaluation~\citep{zheng2023judging} score candidates via trained or prompted models; ACES requires no additional training.
Our $\delta_j$ mirrors the item discrimination index from classical test theory~\citep{lord2008statistical}, adapted to noisy, machine-generated tests; the key difference is our self-referential LOO-AUC mechanism, which replaces the external criterion classical theory assumes.
\section{Conclusion}
\label{sec:conclusion}

We presented ACES~(AUC Consistency Scoring), built on the insight that \emph{test votes should rank, not merely count}.
By evaluating tests against each other via leave-one-out AUC, we break the circular dependency between code and test quality assessment without any external supervision.
The LOO-AUC identity~(Theorem~\ref{lem:loo-identity}) establishes that each test's observable consistency with the ranking is proportional to its latent discriminative power, providing the first provable criterion for distinguishing informative from misleading tests using only the binary pass matrix.
Building on this, ACES-C provides closed-form weights that provably approximate the oracle in expectation, and ACES-O complements it by iteratively optimizing a differentiable LOO-AUC objective without requiring the average-quality assumption.
Both operate solely on the binary pass matrix with negligible overhead, and achieve state-of-the-art Pass@$k$ among execution-only methods; when combined with complementary static-analysis signals, they yield the best overall results across all benchmarks.

The LOO-AUC framework opens several research directions.
Incorporating correlations among LLM-generated tests could yield tighter bounds and stronger weighting schemes.
More broadly, the principle of assessing evaluator quality via internal consistency extends naturally to other noisy-evaluator settings, including LLM-as-judge ensembles, crowdsourced annotation, and verification with process reward models.

\clearpage
\appendix

\section{Proofs and Theoretical Details}
\label{sec:proofs}

This section provides complete proofs for all theoretical results in the main text: the Pass@$k$ bound~(\ref{app:proof-lemma-passk}), the LOO-AUC identity~(\ref{app:proof-loo-identity}), the structure of $c_j(w_{\mathrm{unif}})$~(\ref{app:proof-loo-quality}), and the oracle approximation guarantee~(\ref{app:improvement}).

\subsection{Proof Preliminaries}
\label{app:proof-prelim}

We summarize the notation and model properties used throughout the proofs.

\paragraph{Notation.}
Throughout the appendix, we use the following notation from the main text:
\begin{itemize}
  \item $\alpha_j = P(B_{ij}=1 \mid y_i = 1)$, $\beta_j = P(B_{ij}=1 \mid y_i = 0)$: class-conditional pass rates of test~$t_j$.
  \item $\delta_j = \alpha_j - \beta_j$: discriminative power of test~$t_j$~(Definition~\ref{def:model}).
  \item $C^+ = \{c_i : y_i = 1\}$, $C^- = \{c_i : y_i = 0\}$: sets of correct and incorrect codes.
  \item $n^+ = |C^+|$, $n^- = |C^-|$, $\pi = n^+/n$: counts and fraction of correct codes.
  \item $p_j = P(B_{ij} = 1) = \pi\alpha_j + (1-\pi)\beta_j$: marginal pass rate of test~$t_j$.
\end{itemize}

\paragraph{Model properties.}
As stated in Section~\ref{sec:framework}, codes are sampled independently from the LLM.
Conditioned on the correctness label~$y_i$, each code's test outcomes are identically distributed and mutually independent across tests, since the class-conditional pass rates $\alpha_j$ and $\beta_j$~(Definition~\ref{def:model}) depend only on the test and the label, not on the specific code.
It follows that the true AUC $A(w)$ is the same for every pair $(c^+ \in C^+,\, c^- \in C^-)$ and that two codes from the same class have identically distributed scores.

\paragraph{Hoeffding's inequality.}
Several proofs below rely on the following classical concentration inequality, which we state here for reference.

\begin{lemma}[Hoeffding's inequality~\citep{hoeffding1963probability}]\label{lem:hoeffding}
  If $Z_1, \ldots, Z_m$ are independent random variables with $Z_j \in [a_j, b_j]$, then for any $t > 0$:
  \begin{equation*}
    P\!\left(\sum_{j=1}^{m} Z_j - \E\!\left[\sum_{j=1}^{m} Z_j\right] \leq -t\right) \;\leq\; \exp\!\left(-\frac{2\,t^2}{\sum_{j=1}^{m}(b_j - a_j)^2}\right).
  \end{equation*}
\end{lemma}

\subsection{Proof of Theorem~\ref{thm:hoeffding}~(Pass@k Bound)}
\label{app:proof-lemma-passk}\label{app:proof-thm-hoeffding}

\begin{theorem*}[\ref{thm:hoeffding}, restated]
  Define the \emph{mean signal} and \emph{signal-to-noise ratio} of weights $w$:
  \begin{equation*}
    M(w) \;:=\; \sum_j w_j\delta_j, \qquad R(w) \;:=\; \frac{M(w)^2}{\sum_j w_j^2}.
  \end{equation*}
  For any $w$ with $M(w) > 0$ and any $k \geq 1$:
  \begin{equation*}
    \passk(w) \;\geq\; 1 - \frac{n^-}{k}\,\exp\!\left(-\frac{R(w)}{2}\right),
  \end{equation*}
  where $n^- = n - n^+$ is the number of incorrect codes.
\end{theorem*}

\begin{proof}
The proof proceeds in two steps: first we reduce $\passk$ to a pairwise error probability, then we bound that probability using Hoeffding's inequality.

\paragraph{Step I: Reducing Pass@$k$ to pairwise comparisons.}
Let $c^* = \arg\max_{c \in C^+} s_c(w)$ be the highest-scoring correct code, and define
\begin{equation*}
  R_{c^*} \;=\; |\{c^- \in C^- : s_{c^-}(w) \geq s_{c^*}(w)\}|,
\end{equation*}
i.e., the number of incorrect codes that score at least as high as the best correct code.
If Pass@$k$ fails~(no correct code in the top $k$), then all $k$ top-scoring codes are incorrect and each scores $\geq s_{c^*}(w)$, so $R_{c^*} \geq k$.
We now bound $\E[R_{c^*}]$.
For any fixed correct code $c^+$~(not necessarily $c^*$) and any incorrect code $c^-$, the event $\{s_{c^-} \geq s_{c^*}\}$ implies $\{s_{c^-} \geq s_{c^+}\}$~(since $c^*$ has the highest score among correct codes).
Therefore:
\begin{equation*}
  P(s_{c^-} \geq s_{c^*}) \;\leq\; P(s_{c^-} \geq s_{c^+}).
\end{equation*}
By exchangeability, the right-hand side is the same for every pair $(c^+, c^-)$.
Summing over all $n^-$ incorrect codes:
\begin{equation*}
  \E[R_{c^*}] \;\leq\; n^- \cdot P(s_{c^-} \geq s_{c^+}).
\end{equation*}
Applying Markov's inequality:
\begin{equation*}
  P(\text{Pass@}k \text{ fails}) \;\leq\; P(R_{c^*} \geq k) \;\leq\; \frac{n^- \cdot P(s_{c^-} \geq s_{c^+})}{k}.
\end{equation*}
It remains to bound the pairwise error probability $P(s_{c^-} \geq s_{c^+})$.

\paragraph{Step II: Bounding pairwise error via Hoeffding's inequality.}
Fix a pair $(c^+, c^-)$ and define the per-test vote $h_j = B_{c^+,j} - B_{c^-,j} \in \{-1,0,+1\}$.
Then the score difference is:
\begin{equation*}
  X \;=\; s_{c^+}(w) - s_{c^-}(w) \;=\; \sum_{j=1}^{m} w_j\, h_j.
\end{equation*}
Since tests are independently sampled, conditioned on $(y_{c^+}, y_{c^-}) = (1, 0)$, the test outcomes for $c^+$ and $c^-$ are independent across tests.
This means $\{h_j\}_{j=1}^m$ are mutually independent random variables with:
\begin{equation*}
  \E[h_j] \;=\; \alpha_j - \beta_j \;=\; \delta_j, \qquad \E[X] \;=\; \sum_j w_j \delta_j \;=\; M(w) \;>\; 0.
\end{equation*}
Each $w_j h_j \in [-w_j, w_j]$.
Applying Lemma~\ref{lem:hoeffding} with $Z_j = w_j h_j$~(so $a_j = -w_j$, $b_j = w_j$, $b_j - a_j = 2w_j$) and $t = M(w)$:
\begin{align*}
  P(s_{c^-} \geq s_{c^+})
  &= P(X \leq 0)
  \;\leq\; \exp\!\left(-\frac{2\,M(w)^2}{\sum_j (2w_j)^2}\right) \\
  &= \exp\!\left(-\frac{M(w)^2}{2\sum_j w_j^2}\right)
  \;=\; \exp\!\left(-\frac{R(w)}{2}\right).
\end{align*}

\paragraph{Combining.}
Substituting the Hoeffding bound into Step~I:
\begin{equation*}
  P(\text{Pass@}k \text{ fails}) \;\leq\; \frac{n^-}{k}\,\exp\!\left(-\frac{R(w)}{2}\right),
\end{equation*}
and therefore:
\begin{equation*}
  \passk(w) \;=\; 1 - P(\text{Pass@}k \text{ fails}) \;\geq\; 1 - \frac{n^-}{k}\,\exp\!\left(-\frac{R(w)}{2}\right).
\end{equation*}

\paragraph{Step III: Oracle-optimal weights.}\label{app:proof-oracle}
Since $R(w) = M(w)^2 / \sum_j w_j^2$ is scale-invariant, we may normalize $\sum_j w_j^2 = 1$, so that $R(w) = (\sum_j w_j \delta_j)^2$.
For $w_j \geq 0$, any test with $\delta_j \leq 0$ contributes non-positively to $\sum_j w_j \delta_j$, so the optimum requires $w^*_j = 0$ for all such tests.
Among the remaining tests~($\delta_j > 0$), by Cauchy--Schwarz:
\begin{equation*}
  \sum_{j:\,\delta_j > 0} w_j \delta_j \;\leq\; \sqrt{\sum_{j:\,\delta_j > 0} w_j^2}\;\cdot\; \sqrt{\sum_{j:\,\delta_j > 0} \delta_j^2},
\end{equation*}
with equality iff $w_j \propto \delta_j$.
Under the normalization $\sum_j w_j^2 = 1$, this gives $w^*_j \propto \max(0,\;\delta_j)$ and $R(w^*) = \sum_{j:\delta_j > 0} \delta_j^2$.
Since $\sum_{j:\delta_j \leq 0} \delta_j \leq 0$, we have $\sum_{j:\delta_j > 0} \delta_j \geq m\bar{\delta}$, and by Cauchy--Schwarz with $g = |\{j : \delta_j > 0\}| \leq m$:
\begin{equation*}
  R^* \;=\; \sum_{j:\,\delta_j > 0} \delta_j^2 \;\geq\; \frac{(m\bar{\delta})^2}{m} \;=\; m\bar{\delta}^2 \;=\; R(w_{\mathrm{unif}}). \qedhere
\end{equation*}
\end{proof}

\subsection{Proof of Theorem~\ref{lem:loo-identity}~(LOO-AUC Identity)}
\label{app:proof-loo-identity}

\begin{theorem*}[\ref{lem:loo-identity}, restated]
  For any weights $w$:
  \begin{equation*}
    \E[\looauc_j(w)] - \frac{1}{2} \;=\; c_j(w) \cdot \delta_j,
    \qquad c_j(w) \;:=\; \frac{\pi(1-\pi)}{p_j(1-p_j)}\left(A^{(-j)}(w) - \frac{1}{2}\right),
  \end{equation*}
  where $\pi = n^+/n$ is the fraction of correct codes and $p_j = P(B_{ij} = 1)$ is the marginal pass rate of test~$j$.
\end{theorem*}

\begin{proof}
  \textbf{Setup.}
  Recall the true AUC of the leave-one-out ranking:
  \begin{equation*}
    A^{(-j)}(w) \;=\; P\bigl(S^{(-j)}_{c^+}(w) > S^{(-j)}_{c^-}(w)\bigr) \;+\; \frac{1}{2}\,P\bigl(S^{(-j)}_{c^+}(w) = S^{(-j)}_{c^-}(w)\bigr),
  \end{equation*}
  and $\looauc_j(w) = \AUC(S^{(-j)}, B_{:,j})$, where $S^{(-j)}$ are the leave-one-out scores and $B_{:,j}$ are the pass/fail labels of test~$t_j$.
  By the probabilistic interpretation of AUC, $\looauc_j(w)$ equals the probability that, for a randomly drawn pair of codes~(one from those that pass test~$t_j$ and one from those that fail it), the passer receives a higher leave-one-out score than the failer~(with ties counted as $1/2$).

  Formally, let $c_\mathrm{p}$ and $c_\mathrm{f}$ denote codes drawn uniformly at random from $\{i : B_{ij} = 1\}$ and $\{i : B_{ij} = 0\}$, respectively.
  We emphasize that $c_\mathrm{p}$ and $c_\mathrm{f}$ are defined by the observed pass/fail outcome on test~$t_j$, not by the true labels: unlike $c^+ \in C^+$ and $c^- \in C^-$~(which are truly correct and incorrect), a passer $c_\mathrm{p}$ may be incorrect and a failer $c_\mathrm{f}$ may be correct.
  With this notation:
  \begin{equation*}
    \E[\looauc_j(w)] \;=\; P\bigl(S^{(-j)}_{c_\mathrm{p}} > S^{(-j)}_{c_\mathrm{f}}\bigr) \;+\; \frac{1}{2}\,P\bigl(S^{(-j)}_{c_\mathrm{p}} = S^{(-j)}_{c_\mathrm{f}}\bigr).
  \end{equation*}

  \textbf{Case decomposition.}
  We condition on the true labels $(y_{c_\mathrm{p}}, y_{c_\mathrm{f}})$.
  Let $q_{10}$, $q_{01}$, and $q_{\mathrm{same}}$ denote the probabilities that the pair falls into each case~(these sum to $1$):

  \smallskip
  \begin{center}
  \small
  \begin{tabular}{cccc}
    \toprule
    $y_{c_\mathrm{p}}$ & $y_{c_\mathrm{f}}$ & Prob. & Cond.\ AUC$^*$ \\
    \midrule
    correct & incorrect & $q_{10}$ & $A^{(-j)}(w)$ \\
    incorrect & correct & $q_{01}$ & $1 - A^{(-j)}(w)$ \\
    same class & same class & $q_{\mathrm{same}}$ & $1/2$ \\
    \bottomrule
  \end{tabular}
  \end{center}
  \smallskip

  \noindent $^*$The last column is the conditional AUC of the leave-one-out ranking for each case:
  \begin{equation*}
    P\bigl(S^{(-j)}_{c_\mathrm{p}} > S^{(-j)}_{c_\mathrm{f}} \mid y_{c_\mathrm{p}}, y_{c_\mathrm{f}}\bigr) \;+\; \frac{1}{2}\,P\bigl(S^{(-j)}_{c_\mathrm{p}} = S^{(-j)}_{c_\mathrm{f}} \mid y_{c_\mathrm{p}}, y_{c_\mathrm{f}}\bigr).
  \end{equation*}
  In the first row, $c_\mathrm{p}$ is correct and $c_\mathrm{f}$ is incorrect, so this reduces to $A^{(-j)}(w)$ by definition.
  In the second row the roles are reversed, giving $1 - A^{(-j)}(w)$.
  In the third row, both codes belong to the same class; by exchangeability their leave-one-out scores are identically distributed, so the last column equals $1/2$.
  By the law of total probability, and substituting $q_{\mathrm{same}} = 1 - q_{10} - q_{01}$:
  \begin{align*}
    \E[\looauc_j(w)]
    &= q_{10}\,A^{(-j)} + q_{01}(1 - A^{(-j)}) + (1 - q_{10} - q_{01}) \cdot \frac{1}{2} \\
    &= \frac{1}{2} + q_{10}\!\left(A^{(-j)} - \frac{1}{2}\right) - q_{01}\!\left(A^{(-j)} - \frac{1}{2}\right) \\
    &= \frac{1}{2} + (q_{10} - q_{01})\!\left(A^{(-j)}(w) - \frac{1}{2}\right).
  \end{align*}

  \textbf{Computing $q_{10}$ and $q_{01}$.}
  The quantity $q_{10}$ is the probability that $c_\mathrm{p}$ is correct and $c_\mathrm{f}$ is incorrect.
  Recall that $\pi = n^+/n$ is the fraction of correct codes, $\alpha_j = P(B_{ij}=1 \mid y_i=1)$ and $\beta_j = P(B_{ij}=1 \mid y_i=0)$ are the class-conditional pass rates, and $p_j = \pi\alpha_j + (1-\pi)\beta_j$ is the marginal pass rate.
  By Bayes' rule, the probability that a code passing test~$j$ is correct is:
  \begin{equation*}
    P(y_i = 1 \mid B_{ij} = 1) \;=\; \frac{P(B_{ij} = 1 \mid y_i = 1)\,P(y_i = 1)}{P(B_{ij} = 1)} \;=\; \frac{\pi\alpha_j}{p_j},
  \end{equation*}
  and the probability that a code failing test~$j$ is incorrect is:
  \begin{equation*}
    P(y_i = 0 \mid B_{ij} = 0) \;=\; \frac{P(B_{ij} = 0 \mid y_i = 0)\,P(y_i = 0)}{P(B_{ij} = 0)} \;=\; \frac{(1-\pi)(1-\beta_j)}{1-p_j}.
  \end{equation*}
  Since $c_\mathrm{p}$ and $c_\mathrm{f}$ are drawn independently, $q_{10}$ is the product of these two probabilities:
  \begin{equation*}
    q_{10} \;=\; \frac{\pi\alpha_j}{p_j}\cdot\frac{(1-\pi)(1-\beta_j)}{1-p_j}.
  \end{equation*}
  For $q_{01}$~($c_\mathrm{p}$ is incorrect, $c_\mathrm{f}$ is correct), the probability that a code passing test~$j$ is incorrect is:
  \begin{equation*}
    P(y_i = 0 \mid B_{ij} = 1) \;=\; \frac{P(B_{ij} = 1 \mid y_i = 0)\,P(y_i = 0)}{P(B_{ij} = 1)} \;=\; \frac{(1-\pi)\beta_j}{p_j},
  \end{equation*}
  and the probability that a code failing test~$j$ is correct is:
  \begin{equation*}
    P(y_i = 1 \mid B_{ij} = 0) \;=\; \frac{P(B_{ij} = 0 \mid y_i = 1)\,P(y_i = 1)}{P(B_{ij} = 0)} \;=\; \frac{\pi(1-\alpha_j)}{1-p_j}.
  \end{equation*}
  Therefore:
  \begin{equation*}
    q_{01} \;=\; \frac{(1-\pi)\beta_j}{p_j}\cdot\frac{\pi(1-\alpha_j)}{1-p_j}.
  \end{equation*}

  \textbf{Computing $q_{10} - q_{01}$.}
  Factoring out the common terms:
  \begin{align*}
    q_{10} - q_{01}
    &= \frac{\pi(1-\pi)}{p_j(1-p_j)}\bigl[\alpha_j(1-\beta_j) - \beta_j(1-\alpha_j)\bigr] \\
    &= \frac{\pi(1-\pi)}{p_j(1-p_j)}\bigl[\alpha_j - \beta_j\bigr]
    \;=\; \frac{\pi(1-\pi)\,\delta_j}{p_j(1-p_j)}.
  \end{align*}
  Substituting back into the total probability expression:
  \begin{equation*}
    \E[\looauc_j(w)] - \frac{1}{2} \;=\; (q_{10} - q_{01})\!\left(A^{(-j)}(w) - \frac{1}{2}\right)
    \;=\; \underbrace{\frac{\pi(1-\pi)}{p_j(1-p_j)} \cdot \left(A^{(-j)}(w) - \frac{1}{2}\right)}_{=\;c_j(w)} \cdot\; \delta_j. \qedhere
  \end{equation*}
\end{proof}

\subsection{Proof of Proposition~\ref{thm:loo-quality}~(Structure of \texorpdfstring{$c_j(w_{\mathrm{unif}})$}{c\_j(w\_unif)})}
\label{app:proof-loo-quality}

\begin{proposition*}[\ref{thm:loo-quality}, restated]
  Under Assumption~\ref{asm:main}, $A^{(-j)}(w_{\mathrm{unif}}) > 1/2$ for all~$j$, so $c_j(w_{\mathrm{unif}}) > 0$.
  Moreover, $A^{(-j)}(w_{\mathrm{unif}})$ is approximately constant across~$j$.
\end{proposition*}

\begin{proof}
Under uniform weights $w_k = 1/m$, removing test~$j$ gives $w^{(-j)}_k = 1/m$ for $k \neq j$ and $w^{(-j)}_j = 0$.
The mean signal and sum of squared weights are:
\begin{equation*}
  M^{(-j)} \;=\; \frac{1}{m}\sum_{k \neq j} \delta_k \;=\; \frac{m\bar{\delta} - \delta_j}{m}, \qquad \sum_{k} (w^{(-j)}_k)^2 \;=\; \frac{m-1}{m^2}.
\end{equation*}
Therefore:
\begin{equation*}
  R^{(-j)}(w_{\mathrm{unif}}) \;=\; \frac{M^{(-j)\,2}}{\sum_k (w^{(-j)}_k)^2} \;=\; \frac{(m\bar{\delta} - \delta_j)^2}{m{-}1}.
\end{equation*}
Since $|\delta_j| \leq 1$~(each $\delta_j = \alpha_j - \beta_j$ is a difference of probabilities) and Assumption~\ref{asm:main} implies $m\bar{\delta} > 2\sqrt{m\ln 2} > 1$, the minimum of $R^{(-j)}(w_{\mathrm{unif}})$ over $\delta_j \in [-1,1]$ is attained at $\delta_j = 1$:
\begin{equation*}
  R^{(-j)}(w_{\mathrm{unif}}) \;\geq\; \frac{(m\bar{\delta} - 1)^2}{m{-}1}.
\end{equation*}

We claim this lower bound exceeds $2\ln 2$.
Let $a = m\bar{\delta}$.
From Assumption~\ref{asm:main}, $m\bar{\delta}^2 > 4\ln 2$, so $m < a^2/(4\ln 2)$ and $2(m{-}1)\ln 2 = 2m\ln 2 - 2\ln 2 < a^2/2 - 2\ln 2$.
Therefore:
\begin{equation*}
  (a - 1)^2 - 2(m{-}1)\ln 2
  \;>\; (a - 1)^2 - \frac{a^2}{2} + 2\ln 2
  \;=\; \frac{(a-2)^2}{2} + (2\ln 2 - 1)
  \;>\; 0,
\end{equation*}
where the last step uses $2\ln 2 > 1$.
Hence $R^{(-j)}(w_{\mathrm{unif}}) > 2\ln 2$ for every~$j$.

To connect this to the AUC, recall from Section~\ref{app:proof-thm-hoeffding} that for the score difference $X = s_{c^+}(w) - s_{c^-}(w)$, the Hoeffding bound gives $P(X \leq 0) \leq \exp(-R(w)/2)$.
Since the true AUC satisfies
\begin{equation*}
  1 - A(w) \;=\; P(X < 0) + \frac{1}{2}\,P(X = 0) \;\leq\; P(X \leq 0),
\end{equation*}
we obtain $1 - A(w) \leq \exp(-R(w)/2)$ for any $w$ with $M(w) > 0$.
Applying this to the leave-one-out ranking:
\begin{equation*}
  1 - A^{(-j)}(w_{\mathrm{unif}}) \;\leq\; \exp\!\left(-\frac{R^{(-j)}(w_{\mathrm{unif}})}{2}\right) \;<\; \exp(-\ln 2) \;=\; \frac{1}{2}.
\end{equation*}
Therefore $A^{(-j)}(w_{\mathrm{unif}}) > 1/2$ for every~$j$, which gives $c_j(w_{\mathrm{unif}}) > 0$.

\paragraph{Approximate constancy of $A^{(-j)}$.}
The bound above gives
\begin{equation*}
  1 - A^{(-j)}(w_{\mathrm{unif}}) \;\leq\; \exp\!\left(-\frac{(m\bar{\delta} - 1)^2}{2(m{-}1)}\right)
\end{equation*}
for every~$j$.
Since all $A^{(-j)}(w_{\mathrm{unif}}) \in (1/2,\, 1]$, for any $a \geq b > 1/2$ we have $a - b \leq 1 - b = \max(1 - a,\, 1 - b)$, so
\begin{equation*}
  \left|A^{(-j)} - A^{(-k)}\right|
  \;\leq\; \max\!\left(1 - A^{(-j)},\; 1 - A^{(-k)}\right)
  \;\leq\; \exp\!\left(-\frac{(m\bar{\delta} - 1)^2}{2(m{-}1)}\right).
\end{equation*}
The same reasoning applies to $A(w_{\mathrm{unif}})$, since $R(w_{\mathrm{unif}}) = m\bar{\delta}^2 > 4\ln 2$ gives $1 - A(w_{\mathrm{unif}}) \leq \exp(-m\bar{\delta}^2/2)$, which is also exponentially small.
Therefore all values $A^{(-j)}(w_{\mathrm{unif}})$ and $A(w_{\mathrm{unif}})$ are exponentially close to each other, and $A^{(-j)}(w_{\mathrm{unif}}) - 1/2$ is approximately constant across~$j$.
\end{proof}

\subsection{Oracle Approximation Guarantee}
\label{app:improvement}

\begin{corollary}\label{cor:improvement}
  Under Assumption~\ref{asm:main}, let $a_j = A^{(-j)}(w_{\mathrm{unif}}) - 1/2$ and $\rho = a_{\min}/a_{\max}$.
  Define the \emph{population ACES-C weights} $\bar{w}_j := \max(0,\,\E[q_j])$, where $q_j$ is the quality score from Theorem~\ref{thm:main}.  Since $a_j > 0$, this gives $\bar{w}_j = \pi(1{-}\pi)\,a_j\,\max(0,\,\delta_j)$.  Then:
  \begin{equation*}
    R(\bar{w}) \;\geq\; \rho^2\, R^*, \qquad \rho \;\geq\; 1 - 2\exp\!\left(-\frac{(m\bar{\delta} - 1)^2}{2(m{-}1)}\right),
  \end{equation*}
  where $R^* = \sum_{j:\,\delta_j > 0} \delta_j^2$ is the oracle signal-to-noise ratio~(Theorem~\ref{thm:hoeffding}).
  Since $\rho \to 1$ as $m$ grows, the population weights achieve near-oracle performance.
\end{corollary}

\begin{proof}
  By definition and Theorem~\ref{thm:main}:
  \begin{equation*}
    \bar{w}_j \;=\; \pi(1{-}\pi)\,a_j \max(0,\;\delta_j),
  \end{equation*}
  where $a_j > 0$ by Proposition~\ref{thm:loo-quality}.
  Only tests with $\delta_j > 0$ receive positive population weight, so:
  \begin{equation*}
    M(\bar{w}) \;=\; \sum_{j:\,\delta_j > 0} a_j\,\delta_j^2 \cdot \pi(1{-}\pi), \qquad
    \sum_j \bar{w}_j^2 \;=\; \sum_{j:\,\delta_j > 0} a_j^2\,\delta_j^2 \cdot \pi^2(1{-}\pi)^2.
  \end{equation*}
  Since $R$ is scale-invariant, $\pi(1{-}\pi)$ cancels:
  \begin{equation*}
    R(\bar{w})
    \;=\; \frac{\bigl(\sum_{j:\delta_j>0} a_j\,\delta_j^2\bigr)^2}{\sum_{j:\delta_j>0} a_j^2\,\delta_j^2}
    \;\geq\; \frac{a_{\min}^2\,\bigl(\sum_{j:\delta_j>0} \delta_j^2\bigr)^2}{a_{\max}^2\,\sum_{j:\delta_j>0} \delta_j^2}
    \;=\; \rho^2\,R^*.
  \end{equation*}

  For the bound on $\rho$: by Appendix~\ref{app:proof-loo-quality},
  \begin{equation*}
    1 - A^{(-j)} \;\leq\; \exp\!\left(-\frac{(m\bar{\delta} - 1)^2}{2(m{-}1)}\right) \quad \text{for all } j.
  \end{equation*}
  Since $a_{\max} \leq 1/2$~(because $A^{(-j)} \leq 1$) and $a_{\min} \geq 1/2 - \exp\!\bigl(-(m\bar{\delta} - 1)^2/[2(m{-}1)]\bigr)$:
  \begin{equation*}
    \rho \;=\; \frac{a_{\min}}{a_{\max}} \;\geq\; \frac{1/2 - \exp\!\bigl(-(m\bar{\delta} - 1)^2/[2(m{-}1)]\bigr)}{1/2} \;=\; 1 - 2\exp\!\left(-\frac{(m\bar{\delta} - 1)^2}{2(m{-}1)}\right). \qedhere
  \end{equation*}
\end{proof}

\clearpage
\section{Algorithm Details}
\label{app:algorithms}

This section provides pseudocode for both ACES variants, implementation details, and the combination procedure with $\mathcal{DS}^3$ referenced in Section~\ref{sec:experiments}.

\subsection{ACES-C}
\label{app:alg-aces-c}

Algorithm~\ref{alg:aces} implements ACES-C~(Section~\ref{sec:aces}).
Under uniform initial weights, the leave-one-out scores reduce to unweighted row sums with one column removed; since the AUC is scale-invariant, normalization is unnecessary.
The algorithm has no tunable parameters: it computes the LOO-AUC for each test, applies the pass-rate correction~(Eq.~\ref{eq:aces-weight}), and returns the weighted scores in a single pass.

\begin{algorithm}[H]
  \caption{ACES-C~(Closed-Form Weighting)}
  \label{alg:aces}
  \begin{algorithmic}[1]
    \REQUIRE Pass matrix $B \in \{0,1\}^{n \times m}$
    \ENSURE Code scores $s_1, \ldots, s_n$
    \FOR{$j = 1$ \TO $m$}
      \STATE Compute $S^{(-j)}_i = \sum_{j' \neq j} B_{ij'}$ for all $i$ \hfill \COMMENT{LOO scores; $\propto$ Eq.~\eqref{eq:loo-score} under $w_{\mathrm{unif}}$}
      \STATE Compute $\looauc_j(w_{\mathrm{unif}}) = \AUC(S^{(-j)}, B_{:,j})$ \hfill \COMMENT{AUC is scale-invariant}
    \ENDFOR
    \STATE $p_j \leftarrow \frac{1}{n}\sum_i B_{ij}$ for all $j$ \hfill \COMMENT{empirical pass rate}
    \STATE Set $w_j \leftarrow \max(0,\; \looauc_j(w_{\mathrm{unif}}) - 1/2) \cdot p_j(1-p_j)$ for all $j$ \hfill \COMMENT{Eq.~\eqref{eq:aces-weight}}
    \RETURN $s_i = \sum_j w_j B_{ij}$ for all $i$
  \end{algorithmic}
\end{algorithm}

\subsection{ACES-O}
\label{app:alg-aces-o}

Algorithm~\ref{alg:aces-o} implements ACES-O~(Section~\ref{sec:aces-o}), which maximizes the objective $J(w) = \sum_j w_j\,(\looauc_j(w) - 1/2)$ via gradient ascent.
Since the exact LOO-AUC involves indicator functions and is therefore non-differentiable, we replace it with a logistic surrogate.
For each test~$t_j$, let $\mathcal{P}_j = \{i : B_{ij} = 1\}$ and $\mathcal{N}_j = \{i : B_{ij} = 0\}$ denote the sets of codes that pass and fail~$t_j$, respectively.
The surrogate LOO-AUC is:
\begin{equation}\label{eq:logistic-auc}
  \widehat{\looauc}_j(w) \;=\; \frac{1}{|\mathcal{P}_j|\,|\mathcal{N}_j|}\sum_{i \in \mathcal{P}_j}\sum_{k \in \mathcal{N}_j} \sigma\!\left(\gamma\bigl(S^{(-j)}_i(w) - S^{(-j)}_k(w)\bigr)\right),
\end{equation}
where $\sigma(x) = 1/(1+e^{-x})$ is the logistic function and $\gamma > 0$ is a sharpness parameter controlling the surrogate tightness.
As $\gamma \to \infty$, $\sigma(\gamma x)$ approaches the step function and the surrogate recovers the exact LOO-AUC.
This surrogate is smooth and statistically consistent~\citep{gao2015auc}, enabling gradient computation via automatic differentiation.
The weights are parameterized as $w = \mathrm{softmax}(\ell)$ with learnable logits $\ell \in \R^m$; this ensures $w \in \Delta^m$ throughout optimization.

\begin{algorithm}[H]
  \caption{ACES-O~(Optimized Weighting)}
  \label{alg:aces-o}
  \begin{algorithmic}[1]
    \REQUIRE Pass matrix $B \in \{0,1\}^{n \times m}$, sharpness $\gamma$, learning rate $\eta$, steps $T$
    \ENSURE Code scores $s_1, \ldots, s_n$
    \STATE Initialize logits $\ell \leftarrow \mathbf{0} \in \R^m$
    \FOR{$t = 1$ \TO $T$}
      \STATE $w \leftarrow \mathrm{softmax}(\ell)$
      \FOR{$j = 1$ \TO $m$}
        \STATE $S^{(-j)}_i(w) \leftarrow \sum_{j' \neq j} w_{j'} B_{ij'}$ for all $i$ \hfill \COMMENT{leave-one-out scores}
        \STATE Compute $\widehat{\looauc}_j(w)$ via logistic surrogate~(Eq.~\ref{eq:logistic-auc})
      \ENDFOR
      \STATE $\ell \leftarrow \ell + \eta\, \nabla_\ell \sum_j w_j\,\widehat{\looauc}_j(w)$ \hfill \COMMENT{gradient ascent~(Adam in practice)}
    \ENDFOR
    \STATE $w \leftarrow \mathrm{softmax}(\ell)$
    \RETURN $s_i = \sum_j w_j B_{ij}$ for all $i$
  \end{algorithmic}
\end{algorithm}

The surrogate objective is non-convex because $\widehat{\looauc}_j(w)$ depends on $w$ through the leave-one-out scores, so gradient ascent converges to a stationary point rather than a global maximum.
Initializing from $w_{\mathrm{unif}}$~(i.e., $\ell = \mathbf{0}$) provides a favorable starting point, since Proposition~\ref{thm:loo-quality} guarantees $A^{(-j)}(w_{\mathrm{unif}}) > 1/2$ for all~$j$ under Assumption~\ref{asm:main}: informative tests are up-weighted in early iterates, improving the ranking and further separating informative from misleading tests in subsequent gradient steps.
Across all benchmarks, we observe stable convergence with no restarts required.

\subsection{Implementation Details}
\label{app:impl-details}

For ACES-O, we adopt a two-stage pipeline: first pre-filter to the top-$K$ candidates by majority vote, then optimize weights on this shortlist.
Pre-filtering removes obviously low-quality candidates and reduces the per-step cost from $O(nm^2)$ to $O(Km^2)$.
ACES-C requires no pre-filtering and runs on all $n$ candidates directly.
ACES-O uses Adam with surrogate sharpness $\gamma{=}10$, learning rate $\eta{=}0.01$, and $T{=}90$ steps on HumanEval~($\eta{=}0.05$, $T{=}90$ on HumanEval$^+$ and MBPP), with top-$K{=}24$ pre-filtering~($K{=}32$ on MBPP).

\subsection{Combination with \texorpdfstring{$\mathcal{DS}^3$}{DS3}}
\label{app:ds3-combination}
As described in Section~\ref{sec:main-results}, we combine ACES with $\mathcal{DS}^3$~\citep{Liu_Xue_Ma_Sun_Li_Li_2026}.
We first summarize $\mathcal{DS}^3$'s scoring mechanism, then describe the combined formulation.

\textbf{$\mathcal{DS}^3$ scoring.}
$\mathcal{DS}^3$ first applies a \emph{pre-filtering} stage that sanitizes candidate code~(AST-based extraction of the longest valid code block followed by dependency-driven dead-code elimination) and validates test inputs~(AST-based argument extraction with signature-conformance checking).
This removes malformed candidates and invalid tests, reducing the sets from $n$ to $n' \leq n$ and from $m$ to $m' \leq m$.
On the filtered set, $\mathcal{DS}^3$ computes two signals.
A \emph{static quality vector} $\mathbf{q} \in [0,1]^{n'}$ scores each candidate by combining a normalized Pylint score~(code quality) with a normalized cyclomatic-complexity score~(lower complexity yields a higher score).
A \emph{pairwise consensus matrix} $P \in [0,1]^{n' \times n'}$ captures behavioral similarity:
\[
  P_{ik} \;=\; \frac{1}{m'}\sum_{j=1}^{m'} \mathbf{1}\!\bigl[\operatorname{exec}(c_i, t_j) = \operatorname{exec}(c_k, t_j)\bigr],
\]
where $\operatorname{exec}(c_i, t_j)$ denotes the full output of candidate $c_i$ on test $t_j$.
The final $\mathcal{DS}^3$ score is $\mathbf{q}_* = P \cdot \mathbf{q}$: each candidate receives a consensus-weighted sum of static quality across all candidates, favoring solutions that are both well-structured and behaviorally consistent.

\textbf{Combined formulation.}
ACES and $\mathcal{DS}^3$ capture complementary signals: ACES computes adaptive test weights from the pass matrix to produce per-candidate scores $s_i(w) = \sum_j w_j B_{ij}$, while $\mathcal{DS}^3$ combines static code quality with pairwise output consensus.
We reuse $\mathcal{DS}^3$'s pre-filtering to obtain a cleaned candidate set of size~$n'$, then combine both signals on the filtered candidates via
\begin{equation}\label{eq:ds3-combination}
  \mathbf{r} \;=\; \bigl(\alpha\, I_{n'} + (1-\alpha)\, P\bigr) \cdot \bigl(\beta\, \mathbf{s}(w) + (1-\beta)\, \mathbf{q}\bigr),
\end{equation}
where $I_{n'}$ is the identity matrix and $\alpha, \beta \in [0,1]$.
The coefficient~$\beta$ blends the per-candidate scoring signal: ACES scores~$\mathbf{s}(w)$ at $\beta = 1$ and static quality~$\mathbf{q}$ at $\beta = 0$.
The coefficient~$\alpha$ controls the aggregation: $\alpha = 1$ scores each candidate independently, while $\alpha = 0$ applies pairwise consensus weighting as in $\mathcal{DS}^3$.
Setting $\alpha = \beta = 0$ recovers the original $\mathcal{DS}^3$ score $P \cdot \mathbf{q}$.
Both $\mathbf{s}(w)$ and $\mathbf{q}$ lie in $[0,1]^{n'}$ by construction, so no additional normalization is required.
All $\mathcal{DS}^3$ components~(pre-filtering, static scoring, pairwise consensus matrix, and all associated hyperparameters) use the official implementation and default settings of \citet{Liu_Xue_Ma_Sun_Li_Li_2026}.

\clearpage
\section{Additional Experimental Results}
\label{app:additional-experiments}

This section provides supplementary experiments: an illustrative walkthrough~(\ref{app:example}), additional models and benchmarks~(\ref{app:additional-models}), extended Pass@$k$ curves~(\ref{app:passk-vs-k}), assumption verification~(\ref{app:assumption}), selection vs.\ weighting analysis~(\ref{app:ablation-selection}), sensitivity analyses~(\ref{app:sensitivity-m}--\ref{app:sensitivity-K}), convergence~(\ref{app:convergence}), hyperparameter sensitivity~(\ref{app:sensitivity-hp}), runtime~(\ref{app:runtime}), vote statistics~(\ref{app:noise-stats}), and generation prompts~(\ref{app:prompts}).

\subsection{Illustrative Example}
\label{app:example}

To illustrate how the strength of Assumption~\ref{asm:main} affects the two ACES variants, we construct two $8 \times 10$ instances, each with three correct codes, five incorrect codes, and ten tests ordered by~$\delta_j$.
The \emph{Easy} case~($\bar{\delta}=0.16$) represents a favorable regime, whereas the \emph{Hard} case~($\bar{\delta}=0.04$) contains more strongly misleading tests.
Figure~\ref{fig:overview} summarizes their pass matrices, inferred test weights, and resulting rankings; Tables~\ref{tab:example-easy} and~\ref{tab:example-hard} report the complete numerical data.

\begin{figure}[H]
  \centering
  \includegraphics[width=\textwidth]{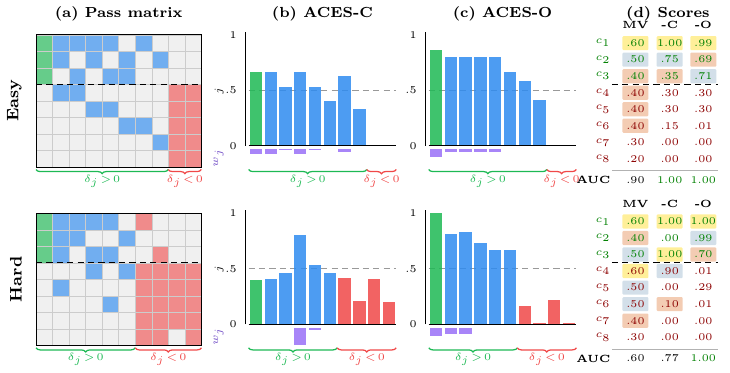}
  \caption{ACES on two constructed $8 \times 10$ instances~(Easy: top; Hard: bottom).
  Test colors: green = perfect, blue = constructive/permissive, red = misleading.
  (a)~Pass matrix. (b),(c)~$\looauc_j$~(upper) and weights $w_j$~(lower, purple). (d)~Scores; gold/silver/bronze mark the top-3 codes.}
  \label{fig:overview}
\end{figure}

Together, the two cases isolate the transition from a favorable regime, where ACES-C's closed-form correction is sufficient, to a difficult regime, where ACES-O's iterative refinement becomes important.
We now trace each case in the same sequence: MV $\to$ ACES-C $\to$ ACES-O.

\textbf{Easy case}~(Table~\ref{tab:example-easy}).
Eight of ten tests are informative~($\delta_j > 0$) and only two are misleading, so the signal-to-noise ratio is favorable.
\emph{MV} assigns uniform weight to all tests and achieves AUC $= 0.90$; however, the lowest correct code $c_3$ ties with three incorrect ones at vote $= 4$, making the top-1 selection unreliable.
\emph{ACES-C} computes LOO-AUC under uniform weights and finds that 6 of the 8 informative tests have $\looauc_j > 1/2$; it assigns non-zero weight to these 6 via Eq.~\eqref{eq:aces-weight}, concentrating weight on the most discriminative tests~($t_1$, $t_2$, $t_4$ each receive $w = 0.25$).
This breaks the MV tie and achieves AUC $= 1.00$: all three correct codes now rank above all five incorrect ones.
\emph{ACES-O} also reaches AUC $= 1.00$, further concentrating weight on the single strongest test~($t_1$, $w = 0.40$).
When test quality is high, the closed-form ACES-C already suffices for perfect separation.

\begin{table}[!t]
  \centering
  \caption{Easy case: $8 \times 10$ pass matrix with 8 informative and 2 misleading tests~($\bar{\delta} = 0.16$).
  Tests are sorted by $\delta_j$ descending. MV AUC: $0.90$; ACES-C AUC: $1.00$; ACES-O AUC: $1.00$.}
  \label{tab:example-easy}
  \vspace{0.3em}
  \small
  \setlength{\tabcolsep}{2.5pt}
  \begin{tabular}{@{}l c ccccccc cc rrr@{}}
    \toprule
    & \multicolumn{1}{c}{\textbf{correct} ($\alpha_j\!=\!1$)} & \multicolumn{9}{c}{\textbf{incorrect} ($\alpha_j\!<\!1$)} & & & \\
    \cmidrule(lr){2-2}\cmidrule(lr){3-11}
    & \multicolumn{1}{c}{\textbf{perf.}} & \multicolumn{7}{c}{\textbf{constructive} ($0 < \delta_j < 1$)} & \multicolumn{2}{c}{\textbf{misleading} ($\delta_j < 0$)} & & & \\
    \cmidrule(lr){2-2}\cmidrule(lr){3-9}\cmidrule(lr){10-11}
    & \cellcolor{cperfect!20}$t_1$ & \cellcolor{cinform!20}$t_2$ & \cellcolor{cinform!20}$t_3$ & \cellcolor{cinform!20}$t_4$ & \cellcolor{cinform!20}$t_5$ & \cellcolor{cinform!20}$t_7$ & \cellcolor{cinform!20}$t_6$ & \cellcolor{cinform!20}$t_8$ & \cellcolor{cmislead!20}$t_9$ & \cellcolor{cmislead!20}$t_{10}$ & \textbf{MV} & \textbf{-C} & \textbf{-O} \\
    \midrule
    \cellcolor{cperfect!15}$\checkmark\, c_1$ & 1 & 1 & 1 & 1 & 1 & 0 & 1 & 0 & 0 & 0 & \cellcolor{gold!40}.60 & \cellcolor{gold!40}1.00 & \cellcolor{gold!40}1.00 \\
    \cellcolor{cperfect!15}$\checkmark\, c_2$ & 1 & 1 & 0 & 1 & 0 & 1 & 0 & 1 & 0 & 0 & \cellcolor{silver!40}.50 & \cellcolor{silver!40}.75 & \cellcolor{bronze!40}.69 \\
    \cellcolor{cperfect!15}$\checkmark\, c_3$ & 1 & 0 & 1 & 0 & 1 & 1 & 0 & 0 & 0 & 0 & \cellcolor{bronze!40}.40 & \cellcolor{bronze!40}.35 & \cellcolor{silver!40}.71 \\
    \cdashline{1-11}
    \cellcolor{cmislead!15}$\times\, c_4$ & 0 & 1 & 1 & 0 & 0 & 0 & 0 & 0 & 1 & 1 & \cellcolor{bronze!40}.40 & .30 & .30 \\
    \cellcolor{cmislead!15}$\times\, c_5$ & 0 & 0 & 0 & 1 & 1 & 0 & 0 & 0 & 1 & 1 & \cellcolor{bronze!40}.40 & .30 & .30 \\
    \cellcolor{cmislead!15}$\times\, c_6$ & 0 & 0 & 0 & 0 & 0 & 1 & 1 & 0 & 1 & 1 & \cellcolor{bronze!40}.40 & .15 & .01 \\
    \cellcolor{cmislead!15}$\times\, c_7$ & 0 & 0 & 0 & 0 & 0 & 0 & 0 & 1 & 1 & 1 & .30 & .00 & .00 \\
    \cellcolor{cmislead!15}$\times\, c_8$ & 0 & 0 & 0 & 0 & 0 & 0 & 0 & 0 & 1 & 1 & .20 & .00 & .00 \\
    \midrule
    $\alpha_j$ & 1.00 & .67 & .67 & .67 & .67 & .67 & .33 & .33 & .00 & .00 & & & \\
    $\beta_j$ & .00 & .20 & .20 & .20 & .20 & .20 & .20 & .20 & 1.00 & 1.00 & & & \\
    $\delta_j$ & $+1.0$ & $+.47$ & $+.47$ & $+.47$ & $+.47$ & $+.47$ & $+.13$ & $+.13$ & $-1.0$ & $-1.0$ & & & \\
    \addlinespace[2pt]
    LOO$_{\text{C}}$ & .67 & .67 & .53 & .67 & .53 & .40 & .63 & .33 & .00 & .00 & & & \\
    $w_j^{\text{C}}$ & .25 & .25 & .05 & .25 & .05 & .00 & .15 & .00 & .00 & .00 & & & \\
    LOO$_{\text{O}}$ & .87 & .80 & .80 & .80 & .80 & .67 & .58 & .42 & .00 & .00 & & & \\
    $w_j^{\text{O}}$ & .40 & .14 & .15 & .14 & .15 & .00 & .00 & .00 & .00 & .00 & & & \\
    \bottomrule
  \end{tabular}
\end{table}
 
\textbf{Hard case}~(Table~\ref{tab:example-hard}).
By contrast, six tests are informative and four are misleading, with all four misleading tests having strong discriminative power in the wrong direction~($|\delta_j| \geq 0.67$).
\emph{MV} achieves only AUC $= 0.60$: the incorrect code $c_4$ receives vote $= 6$, tying the best correct code $c_1$.
\emph{ACES-C} faces a challenge: the strong misleading tests corrupt the uniform-weight ranking, causing only 2 of 6 informative tests~($t_4$ and $t_5$) to have $\looauc_j > 1/2$.
Since ACES-C computes weights in closed form from these initial estimates, it can only assign weight to these two tests.
This improves AUC to $0.77$, but the incorrect $c_4$ remains ranked third.
\emph{ACES-O} overcomes this limitation through iterative co-evolution: improved weights produce a better ranking, which raises the LOO-AUC of initially misidentified tests, enabling further weight refinement.
At convergence, all six informative tests have $\looauc_j > 1/2$, and ACES-O concentrates weight on the three strongest~($t_1$, $t_3$, $t_2$), achieving AUC $= 1.00$.
This demonstrates the value of iterative optimization when misleading tests are prevalent.

\begin{table}[!t]
  \centering
  \caption{Hard case: $8 \times 10$ pass matrix with 6 informative and 4 misleading tests~($\bar{\delta} = 0.04$).
  Under uniform weights, only 2/6 informative tests have $\looauc_j > 1/2$; ACES-O's co-evolution recovers all 6.
  MV AUC: $0.60$; ACES-C AUC: $0.77$; ACES-O AUC: $1.00$.}
  \label{tab:example-hard}
  \vspace{0.3em}
  \small
  \setlength{\tabcolsep}{2.5pt}
  \begin{tabular}{@{}l c c cccc cccc rrr@{}}
    \toprule
    & \multicolumn{2}{c}{\textbf{correct} ($\alpha_j\!=\!1$)} & \multicolumn{8}{c}{\textbf{incorrect} ($\alpha_j\!<\!1$)} & & & \\
    \cmidrule(lr){2-3}\cmidrule(lr){4-11}
    & \multicolumn{1}{c}{\textbf{perf.}} & \multicolumn{1}{c}{\textbf{perm.}} & \multicolumn{4}{c}{\textbf{constructive}} & \multicolumn{4}{c}{\textbf{misleading} ($\delta_j < 0$)} & & & \\
    \cmidrule(lr){2-2}\cmidrule(lr){3-3}\cmidrule(lr){4-7}\cmidrule(lr){8-11}
    & \cellcolor{cperfect!20}$t_1$ & \cellcolor{cinform!20}$t_2$ & \cellcolor{cinform!20}$t_3$ & \cellcolor{cinform!20}$t_4$ & \cellcolor{cinform!20}$t_5$ & \cellcolor{cinform!20}$t_6$ & \cellcolor{cmislead!20}$t_8$ & \cellcolor{cmislead!20}$t_9$ & \cellcolor{cmislead!20}$t_7$ & \cellcolor{cmislead!20}$t_{10}$ & \textbf{MV} & \textbf{-C} & \textbf{-O} \\
    \midrule
    \cellcolor{cperfect!15}$\checkmark\, c_1$ & 1 & 1 & 1 & 1 & 1 & 0 & 1 & 0 & 0 & 0 & \cellcolor{gold!40}.60 & \cellcolor{gold!40}1.00 & \cellcolor{gold!40}1.00 \\
    \cellcolor{cperfect!15}$\checkmark\, c_2$ & 1 & 1 & 1 & 0 & 0 & 1 & 0 & 0 & 0 & 0 & \cellcolor{bronze!40}.40 & .00 & \cellcolor{silver!40}.99 \\
    \cellcolor{cperfect!15}$\checkmark\, c_3$ & 1 & 1 & 0 & 1 & 1 & 0 & 0 & 1 & 0 & 0 & \cellcolor{silver!40}.50 & \cellcolor{gold!40}1.00 & \cellcolor{bronze!40}.70 \\
    \cdashline{1-11}
    \cellcolor{cmislead!15}$\times\, c_4$ & 0 & 0 & 0 & 1 & 0 & 1 & 1 & 1 & 1 & 1 & \cellcolor{gold!40}.60 & \cellcolor{silver!40}.90 & .01 \\
    \cellcolor{cmislead!15}$\times\, c_5$ & 0 & 1 & 0 & 0 & 0 & 0 & 1 & 1 & 1 & 1 & \cellcolor{silver!40}.50 & .00 & .29 \\
    \cellcolor{cmislead!15}$\times\, c_6$ & 0 & 0 & 0 & 0 & 1 & 0 & 1 & 1 & 1 & 1 & \cellcolor{silver!40}.50 & \cellcolor{bronze!40}.10 & .01 \\
    \cellcolor{cmislead!15}$\times\, c_7$ & 0 & 0 & 0 & 0 & 0 & 0 & 1 & 1 & 1 & 1 & \cellcolor{bronze!40}.40 & .00 & .00 \\
    \cellcolor{cmislead!15}$\times\, c_8$ & 0 & 0 & 0 & 0 & 0 & 0 & 1 & 1 & 0 & 1 & .30 & .00 & .00 \\
    \midrule
    $\alpha_j$ & 1.00 & 1.00 & .67 & .67 & .67 & .33 & .33 & .33 & .00 & .00 & & & \\
    $\beta_j$ & .00 & .20 & .00 & .20 & .20 & .20 & 1.00 & 1.00 & .80 & 1.00 & & & \\
    $\delta_j$ & $+1.0$ & $+.80$ & $+.67$ & $+.47$ & $+.47$ & $+.13$ & $-.67$ & $-.67$ & $-.80$ & $-1.0$ & & & \\
    \addlinespace[2pt]
    LOO$_{\text{C}}$ & .40 & .41 & .46 & .80 & .53 & .46 & .42 & .21 & .41 & .20 & & & \\
    $w_j^{\text{C}}$ & .00 & .00 & .00 & .90 & .10 & .00 & .00 & .00 & .00 & .00 & & & \\
    LOO$_{\text{O}}$ & 1.00 & .81 & .83 & .73 & .67 & .67 & .17 & .00 & .22 & .00 & & & \\
    $w_j^{\text{O}}$ & .40 & .29 & .30 & .01 & .00 & .00 & .00 & .00 & .00 & .00 & & & \\
    \bottomrule
  \end{tabular}
\end{table}
 
Taken together, these controlled cases make the complementarity between ACES-C and ACES-O explicit.
We next move from mechanism-focused examples to broader evaluations across generation models and benchmarks.

\subsection{Additional Generation Models and Benchmarks}
\label{app:additional-models}

We evaluate ACES on three additional code generation models beyond GPT-3.5-Turbo: Qwen2.5-Coder-7B and 14B~\citep{hui2024qwen2}, and DeepSeek-Coder-V2-16B~\citep{guo2024deepseek}.
In addition to HumanEval and HumanEval$^+$, we include LeetCodeDataset~\citep{xia2025leetcodedataset} as a harder benchmark.
For each model, we generate candidates and test cases following the same protocol as Section~\ref{sec:exp-setup}.
Table~\ref{tab:additional-models} reports results on all tasks; Table~\ref{tab:additional-models-nontrivial} reports results on non-trivial tasks~($0 < n^+ < n$), excluding tasks where all or no candidates are correct.

\begin{table}[H]
  \centering
  \caption{Pass@$k$ (\%) with additional generation models on all tasks.
  Each model generates its own candidates and test cases following the protocol of Section~\ref{sec:exp-setup}.
  Subscripts denote the change from each model's \gc{zero-shot} baseline.
  \textbf{Bold}/\underline{underline}: best/second-best among reranking methods within each model.
  \colorbox{ourbg}{Shaded}: ours.}
  \label{tab:additional-models}
  \vspace{0.5em}
  \small
  \setlength{\tabcolsep}{2.2pt}
  \begin{tabular}{@{}l lll lll lll@{}}
    \toprule
    & \multicolumn{3}{c}{\textbf{HumanEval}} & \multicolumn{3}{c}{\textbf{HumanEval$^+$}} & \multicolumn{3}{c}{\textbf{LeetCodeDataset}} \\
    \cmidrule(lr){2-4} \cmidrule(lr){5-7} \cmidrule(lr){8-10}
    \textbf{Method} & Pass@1 & Pass@2 & Pass@5 & Pass@1 & Pass@2 & Pass@5 & Pass@1 & Pass@2 & Pass@5 \\

    \midrule
    \rowcolor{black!5}
    \multicolumn{10}{@{}l}{\textit{Qwen2.5-Coder-7B}} \\
    \gc{Zero-shot} & \gc{81.69} & \gc{87.21} & \gc{91.27} & \gc{75.63} & \gc{81.38} & \gc{85.62} & \gc{14.64} & \gc{18.82} & \gc{24.50} \\
    Majority Voting & 89.02\up{7.3} & \underline{90.24\up{3.0}} & \underline{90.24\dn{1.0}} & 80.49\up{4.9} & \underline{84.15\up{2.8}} & \underline{84.15\dn{1.5}} & 14.04\dn{0.6} & 15.79\dn{3.0} & 19.74\dn{4.8} \\
    \rowcolor{ourbg}
    \textbf{ACES-C} & \underline{89.63\up{7.9}} & \underline{90.24\up{3.0}} & \underline{90.24\dn{1.0}} & \underline{81.71\up{6.1}} & \underline{84.15\up{2.8}} & \textbf{84.76\dn{0.9}} & \textbf{17.54\up{2.9}} & \textbf{18.86\up{0.0}} & \underline{20.18\dn{4.3}} \\
    \rowcolor{ourbg}
    \textbf{ACES-O} & \textbf{90.24\up{8.6}} & \textbf{90.85\up{3.6}} & \textbf{90.85\dn{0.4}} & \textbf{82.93\up{7.3}} & \textbf{84.76\up{3.4}} & \textbf{84.76\dn{0.9}} & \underline{15.35\up{0.7}} & \underline{16.67\dn{2.2}} & \textbf{20.61\dn{3.9}} \\

    \midrule
    \rowcolor{black!5}
    \multicolumn{10}{@{}l}{\textit{Qwen2.5-Coder-14B}} \\
    \gc{Zero-shot} & \gc{87.74} & \gc{91.45} & \gc{94.50} & \gc{79.40} & \gc{82.97} & \gc{85.96} & \gc{20.54} & \gc{24.94} & \gc{30.48} \\
    Majority Voting & 92.07\up{4.3} & \underline{92.68\up{1.2}} & \underline{92.68\dn{1.8}} & \underline{84.76\up{5.4}} & \textbf{85.37\up{2.4}} & \textbf{87.80\up{1.8}} & \underline{23.25\up{2.7}} & \underline{26.32\up{1.4}} & \underline{31.14\up{0.7}} \\
    \rowcolor{ourbg}
    \textbf{ACES-C} & \underline{92.68\up{4.9}} & \underline{92.68\up{1.2}} & \textbf{93.29\dn{1.2}} & \underline{84.76\up{5.4}} & \underline{84.76\up{1.8}} & 86.59\up{0.6} & \textbf{28.07\up{7.5}} & \textbf{29.82\up{4.9}} & \textbf{32.89\up{2.4}} \\
    \rowcolor{ourbg}
    \textbf{ACES-O} & \textbf{93.29\up{5.6}} & \textbf{93.29\up{1.8}} & \textbf{93.29\dn{1.2}} & \textbf{85.37\up{6.0}} & \textbf{85.37\up{2.4}} & \underline{87.20\up{1.2}} & 21.49\up{1.0} & 24.12\dn{0.8} & 28.95\dn{1.5} \\

    \midrule
    \rowcolor{black!5}
    \multicolumn{10}{@{}l}{\textit{DeepSeek-Coder-V2-16B}} \\
    \gc{Zero-shot} & \gc{77.74} & \gc{82.26} & \gc{86.50} & \gc{71.05} & \gc{75.30} & \gc{79.03} & \gc{20.68} & \gc{24.25} & \gc{28.14} \\
    Majority Voting & 84.76\up{7.0} & 85.98\up{3.7} & \underline{86.59\up{0.1}} & 76.83\up{5.8} & 77.44\up{2.1} & \underline{78.05\dn{1.0}} & 17.11\dn{3.6} & 20.18\dn{4.1} & 24.12\dn{4.0} \\
    \rowcolor{ourbg}
    \textbf{ACES-C} & \textbf{87.20\up{9.5}} & \textbf{87.20\up{4.9}} & \textbf{87.20\up{0.7}} & \textbf{78.66\up{7.6}} & \textbf{78.66\up{3.4}} & \textbf{78.66\dn{0.4}} & \textbf{25.00\up{4.3}} & \textbf{25.88\up{1.6}} & \textbf{26.75\dn{1.4}} \\
    \rowcolor{ourbg}
    \textbf{ACES-O} & \underline{85.98\up{8.2}} & \underline{86.59\up{4.3}} & \textbf{87.20\up{0.7}} & \underline{78.05\up{7.0}} & \underline{78.05\up{2.8}} & \textbf{78.66\dn{0.4}} & \underline{19.74\dn{0.9}} & \underline{21.05\dn{3.2}} & \underline{24.56\dn{3.6}} \\
    \bottomrule
  \end{tabular}
\end{table}
 
\begin{table}[H]
  \centering
  \caption{Pass@$k$ (\%) with additional generation models on non-trivial tasks~($0 < n^+ < n$), excluding tasks where all or no candidates are correct.
  Subscripts denote the change from each model's \gc{zero-shot} baseline on the same set of tasks.
  \textbf{Bold}/\underline{underline}: best/second-best among reranking methods within each model.
  \colorbox{ourbg}{Shaded}: ours.}
  \label{tab:additional-models-nontrivial}
  \vspace{0.5em}
  \small
  \setlength{\tabcolsep}{2.1pt}
  \begin{tabular}{@{}l lll lll lll@{}}
    \toprule
    & \multicolumn{3}{c}{\textbf{HumanEval}} & \multicolumn{3}{c}{\textbf{HumanEval$^+$}} & \multicolumn{3}{c}{\textbf{LeetCodeDataset}} \\
    \cmidrule(lr){2-4} \cmidrule(lr){5-7} \cmidrule(lr){8-10}
    \textbf{Method} & Pass@1 & Pass@2 & Pass@5 & Pass@1 & Pass@2 & Pass@5 & Pass@1 & Pass@2 & Pass@5 \\

    \midrule
    \rowcolor{black!5}
    \multicolumn{10}{@{}l}{\textit{Qwen2.5-Coder-7B}} \\
    \gc{Zero-shot} & \gc{76.98} & \gc{86.03} & \gc{92.68} & \gc{75.03} & \gc{84.46} & \gc{91.42} & \gc{31.02} & \gc{42.24} & \gc{57.49} \\
    Majority Voting & 89.00\up{12.0} & \underline{91.00\up{5.0}} & \underline{91.00\dn{1.7}} & 83.00\up{8.0} & \underline{89.00\up{4.5}} & \underline{89.00\dn{2.4}} & 29.41\dn{1.6} & 34.12\dn{8.1} & 44.71\dn{12.8} \\
    \rowcolor{ourbg}
    \textbf{ACES-C} & \underline{90.00\up{13.0}} & \underline{91.00\up{5.0}} & \underline{91.00\dn{1.7}} & \underline{85.00\up{10.0}} & \underline{89.00\up{4.5}} & \textbf{90.00\dn{1.4}} & \textbf{38.82\up{7.8}} & \textbf{42.35\up{0.1}} & \underline{45.88\dn{11.6}} \\
    \rowcolor{ourbg}
    \textbf{ACES-O} & \textbf{91.00\up{14.0}} & \textbf{92.00\up{6.0}} & \textbf{92.00\dn{0.7}} & \textbf{87.00\up{12.0}} & \textbf{90.00\up{5.5}} & \textbf{90.00\dn{1.4}} & \underline{32.94\up{1.9}} & \underline{36.47\dn{5.8}} & \textbf{47.06\dn{10.4}} \\

    \midrule
    \rowcolor{black!5}
    \multicolumn{10}{@{}l}{\textit{Qwen2.5-Coder-14B}} \\
    \gc{Zero-shot} & \gc{68.41} & \gc{80.36} & \gc{90.16} & \gc{67.05} & \gc{77.32} & \gc{85.91} & \gc{32.83} & \gc{42.87} & \gc{55.50} \\
    Majority Voting & 82.35\up{13.9} & \underline{84.31\up{4.0}} & \underline{84.31\dn{5.9}} & \underline{82.46\up{15.4}} & \textbf{84.21\up{6.9}} & \textbf{91.23\up{5.3}} & \underline{39.00\up{6.2}} & \underline{46.00\up{3.1}} & \underline{57.00\up{1.5}} \\
    \rowcolor{ourbg}
    \textbf{ACES-C} & \underline{84.31\up{15.9}} & \underline{84.31\up{4.0}} & \textbf{86.27\dn{3.9}} & \underline{82.46\up{15.4}} & \underline{82.46\up{5.1}} & 87.72\up{1.8} & \textbf{50.00\up{17.2}} & \textbf{54.00\up{11.1}} & \textbf{61.00\up{5.5}} \\
    \rowcolor{ourbg}
    \textbf{ACES-O} & \textbf{86.27\up{17.9}} & \textbf{86.27\up{5.9}} & \textbf{86.27\dn{3.9}} & \textbf{84.21\up{17.2}} & \textbf{84.21\up{6.9}} & \underline{89.47\up{3.6}} & 35.00\up{2.2} & 41.00\dn{1.9} & 52.00\dn{3.5} \\

    \midrule
    \rowcolor{black!5}
    \multicolumn{10}{@{}l}{\textit{DeepSeek-Coder-V2-16B}} \\
    \gc{Zero-shot} & \gc{64.38} & \gc{75.63} & \gc{86.15} & \gc{61.77} & \gc{72.63} & \gc{82.19} & \gc{39.14} & \gc{49.74} & \gc{61.25} \\
    Majority Voting & 81.82\up{17.4} & 84.85\up{9.2} & \underline{86.36\up{0.2}} & 76.56\up{14.8} & 78.12\up{5.5} & \underline{79.69\dn{2.5}} & 28.57\dn{10.6} & 37.66\dn{12.1} & 49.35\dn{11.9} \\
    \rowcolor{ourbg}
    \textbf{ACES-C} & \textbf{87.88\up{23.5}} & \textbf{87.88\up{12.3}} & \textbf{87.88\up{1.7}} & \textbf{81.25\up{19.5}} & \textbf{81.25\up{8.6}} & \textbf{81.25\dn{0.9}} & \textbf{51.95\up{12.8}} & \textbf{54.55\up{4.8}} & \textbf{57.14\dn{4.1}} \\
    \rowcolor{ourbg}
    \textbf{ACES-O} & \underline{84.85\up{20.5}} & \underline{86.36\up{10.7}} & \textbf{87.88\up{1.7}} & \underline{79.69\up{17.9}} & \underline{79.69\up{7.1}} & \textbf{81.25\dn{0.9}} & \underline{36.36\dn{2.8}} & \underline{40.26\dn{9.5}} & \underline{50.65\dn{10.6}} \\
    \bottomrule
  \end{tabular}
\end{table}
 
Three findings emerge from these results.

\textbf{At least one ACES variant improves Pass@1 over Majority Voting on every model--benchmark pair.}
The improvement is most pronounced at $k{=}1$, where ranking quality is most critical.
On HumanEval, the best ACES variant gains $+1.2$ to $+2.4$ points over MV across the three models; on HumanEval$^+$, the gain ranges from $+0.6$ to $+2.4$ points.
The two variants exhibit complementary strengths: ACES-O leads ACES-C on HumanEval and HumanEval$^+$ for both Qwen models, whereas ACES-C leads for DeepSeek-Coder-V2-16B and on LeetCodeDataset across all three models.

\textbf{On LeetCodeDataset, ACES-C is consistently strongest and corrects MV failures on the harder pools.}
Majority Voting falls below zero-shot for Qwen2.5-Coder-7B~(14.04\% vs.\ 14.64\%) and DeepSeek-Coder-V2-16B~(17.11\% vs.\ 20.68\%), while improving over zero-shot for Qwen2.5-Coder-14B~(23.25\% vs.\ 20.54\%).
ACES-C reaches 17.54\%, 28.07\%, and 25.00\% Pass@1, respectively, exceeding both MV and zero-shot for every model; its gain over MV grows from $+3.5$ points on Qwen2.5-Coder-7B to $+4.8$ on Qwen2.5-Coder-14B and $+7.9$ on DeepSeek-Coder-V2-16B.
ACES-O is less stable: it improves over MV for Qwen2.5-Coder-7B and DeepSeek-Coder-V2-16B, but trails MV for Qwen2.5-Coder-14B~(21.49\% vs.\ 23.25\%).
The consistent ACES-C lead, together with ACES-O's regression on Qwen2.5-Coder-14B, shows that the closed-form variant is more robust on these LeetCodeDataset candidate pools, whereas iterative optimization does not uniformly improve over MV in this regime.

\textbf{Non-trivial tasks amplify the improvements.}
When trivial tasks are excluded~(Table~\ref{tab:additional-models-nontrivial}), all methods operate on tasks where the ranking actually matters, and the ACES advantage becomes more visible.
On LeetCodeDataset non-trivial tasks, ACES-C improves over MV by $+9.4$ points~(Qwen2.5-Coder-7B), $+11.0$ points~(Qwen2.5-Coder-14B), and $+23.4$ points~(DeepSeek-Coder-V2-16B) at Pass@1.
On HumanEval non-trivial tasks, the corresponding gains are $+1.0$, $+2.0$, and $+6.1$ points, respectively.

\subsection{Pass@k vs.\ k}
\label{app:passk-vs-k}

Figure~\ref{fig:passk-vs-k} extends the main results~(combined with $\mathcal{DS}^3$ pre-filtering) to $k = 1, \ldots, 20$.
Both ACES variants substantially outperform Majority Voting across all benchmarks and $k$, with the largest gains at small~$k$ where ranking quality matters most.
On HumanEval and HumanEval$^+$, ACES-C + $\mathcal{DS}^3$ leads at small~$k$~(e.g., $+2.44\%$ Pass@1 on HumanEval$^+$), reflecting its closed-form weighting that concentrates on a few highly discriminative tests for sharp top-1 discrimination; the gap between the two variants narrows as~$k$ grows and they converge at larger~$k$.
On MBPP, ACES-O + $\mathcal{DS}^3$ slightly leads at small~$k$, consistent with MBPP's higher fraction of misleading tests~(Section~\ref{sec:analysis}) making iterative optimization more valuable; beyond $k{=}2$ the two variants are nearly identical.
The overall improvement scales with benchmark difficulty: ${\sim}6$--$8\%$ on MBPP~(most misleading tests) vs.\ ${\sim}3$--$5\%$ on HumanEval, confirming the greater benefit of principled test weighting when test quality is more heterogeneous.

\begin{figure}[H]
  \centering
  \includegraphics[width=\textwidth]{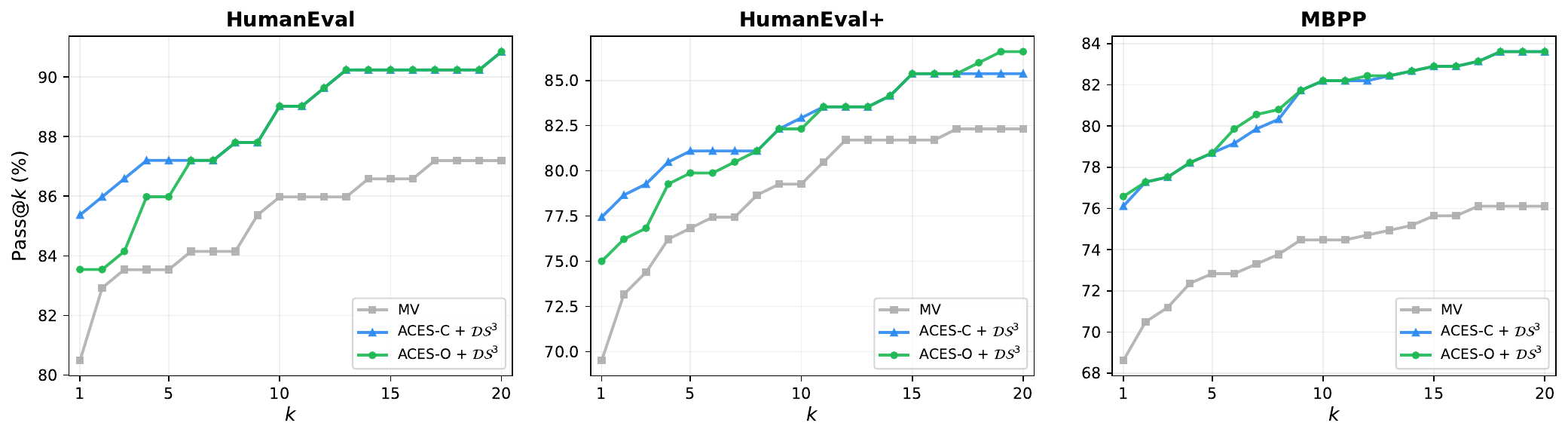}
  \caption{Pass@$k$ vs.\ $k$~($k = 1, \ldots, 20$) on all three benchmarks, combined with $\mathcal{DS}^3$ pre-filtering.
  ACES-C + $\mathcal{DS}^3$ leads at small $k$ on HumanEval/HumanEval$^+$; the two variants converge at larger $k$. On MBPP they perform comparably throughout.}
  \label{fig:passk-vs-k}
\end{figure}

\subsection{Assumption Satisfaction Analysis}
\label{app:assumption}

We verify Assumption~\ref{asm:main} empirically on all three benchmarks.
For each task, we remove trivial tasks~(all codes correct or all incorrect) and constant-column tests~(all codes pass or all fail), then compute $\bar{\delta} = (1/m)\sum_j \delta_j$ over the remaining $m$ non-constant tests.
Assumption~\ref{asm:main} formally requires $\bar{\delta} > 2\sqrt{\ln 2 / m}$, where the threshold is task-specific since each task has a different number of non-constant tests~$m$.
Table~\ref{tab:assumption-summary} reports the formal satisfaction rate.
The figures in this section bin tasks by $\bar{\delta}$ using the simplified threshold $\bar{\delta} > 0$ for visual clarity.
Tasks are binned with width $0.1$~(left-open right-closed) and further divided into three equal-width regions~(\emph{Hard} / \emph{Middle} / \emph{Easy}).

\begin{table}[H]
  \centering
  \caption{Assumption~\ref{asm:main} satisfaction across benchmarks. The threshold $2\sqrt{\ln 2/m}$ is evaluated per task using its number of non-constant tests~$m$.}
  \label{tab:assumption-summary}
  \begin{tabular}{lccc}
    \toprule
    Benchmark & Non-trivial tasks & $\bar{\delta} > 2\sqrt{\ln 2/m}$ (\%) & $\bar{\delta}$ range \\
    \midrule
    HumanEval      & 118 & 83.1 & $[-0.42,\; 1.00]$ \\
    HumanEval$^+$  & 123 & 82.1 & $[-0.42,\; 1.00]$ \\
    MBPP           & 239 & 71.1 & $[-0.90,\; 1.00]$ \\
    \bottomrule
  \end{tabular}
\end{table}

The assumption is satisfied for the majority of non-trivial tasks on all benchmarks.
MBPP has a higher proportion of hard tasks, consistent with its greater fraction of misleading tests.

Figures~\ref{fig:app-assumption-humaneval}--\ref{fig:app-assumption-mbpp} show the full per-benchmark results for Pass@$k$~($k \in \{1, 2, 5\}$), extending the MBPP Pass@1 analysis in Section~\ref{sec:analysis} to all benchmarks and $k$ values.
The key findings from the main text generalize consistently across all three benchmarks and $k$ values:

\textit{(i) The Middle region is where ACES provides the most value.}
In the Easy region~(high $\bar{\delta}$), all methods achieve near-perfect pass rates since most tests are informative and even uniform weighting works well.
The performance gains of ACES concentrate in the Middle region, where informative and misleading tests coexist and principled test weighting becomes critical.

\textit{(ii) ACES-O and ACES-C have complementary strengths across difficulty regimes.}
When the assumption is not well satisfied~(low $\bar{\delta}$), ACES-O's iterative optimization is more effective at identifying and up-weighting the few reliable tests, giving it a clear advantage over ACES-C.
When the assumption is well satisfied~(high $\bar{\delta}$), ACES-C's closed-form solution is already near-optimal, matching or approaching ACES-O without requiring iterative optimization.

\textit{(iii) As $k$ increases, the gap between methods narrows}, since selecting more candidates reduces the sensitivity to ranking quality at the top.
This pattern is consistent across all benchmarks.

\begin{figure}[H]
  \centering
  \begin{subfigure}[b]{0.33\textwidth}
    \centering
    \includegraphics[width=\textwidth]{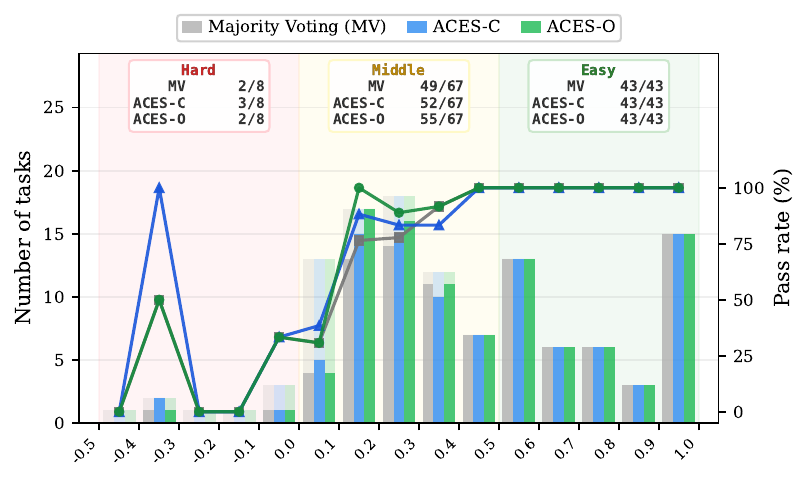}
    \caption{Pass@1}
  \end{subfigure}%
  \hfill
  \begin{subfigure}[b]{0.33\textwidth}
    \centering
    \includegraphics[width=\textwidth]{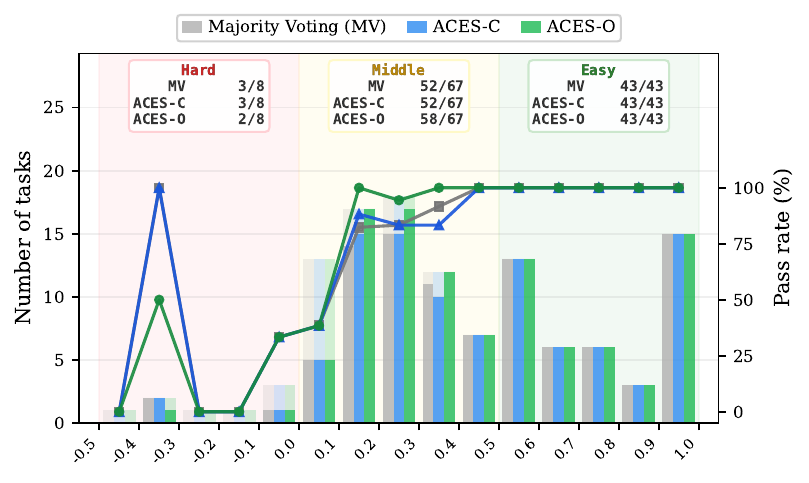}
    \caption{Pass@2}
  \end{subfigure}%
  \hfill
  \begin{subfigure}[b]{0.33\textwidth}
    \centering
    \includegraphics[width=\textwidth]{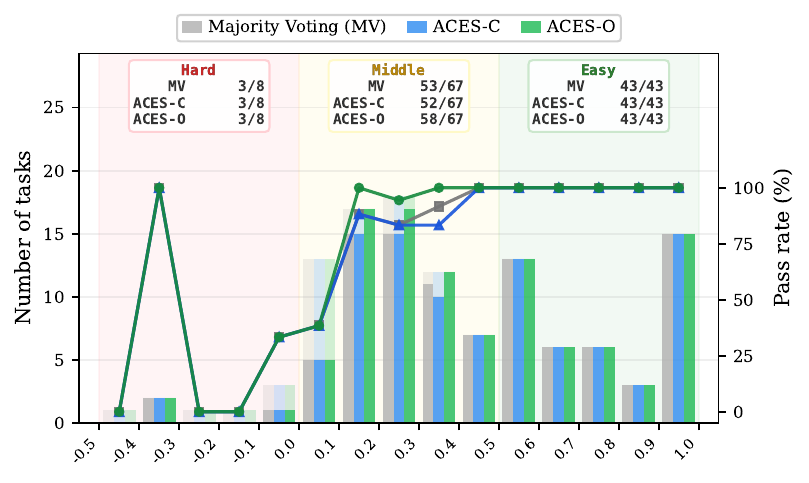}
    \caption{Pass@5}
  \end{subfigure}
  \caption{Assumption~\ref{asm:main} analysis on HumanEval~(tasks binned by $\bar{\delta}$).
  ACES-O leads in the Middle region across all $k$; the gap narrows as $k$ increases.}
  \label{fig:app-assumption-humaneval}
\end{figure}

\begin{figure}[H]
  \centering
  \begin{subfigure}[b]{0.33\textwidth}
    \centering
    \includegraphics[width=\textwidth]{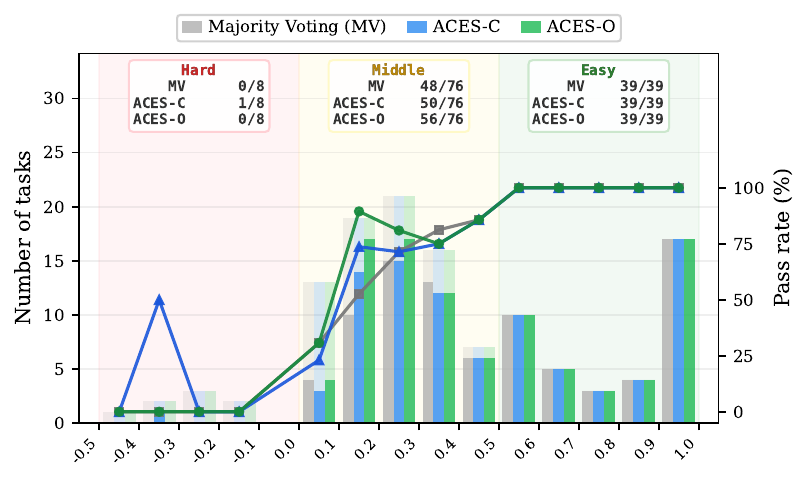}
    \caption{Pass@1}
  \end{subfigure}%
  \hfill
  \begin{subfigure}[b]{0.33\textwidth}
    \centering
    \includegraphics[width=\textwidth]{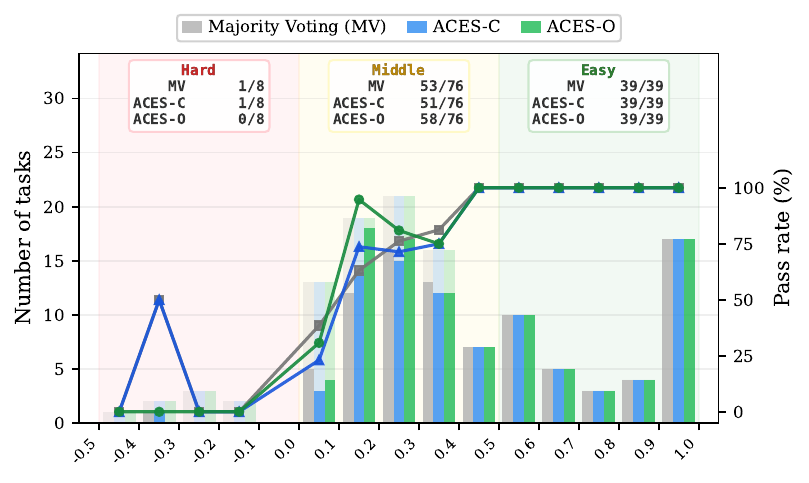}
    \caption{Pass@2}
  \end{subfigure}%
  \hfill
  \begin{subfigure}[b]{0.33\textwidth}
    \centering
    \includegraphics[width=\textwidth]{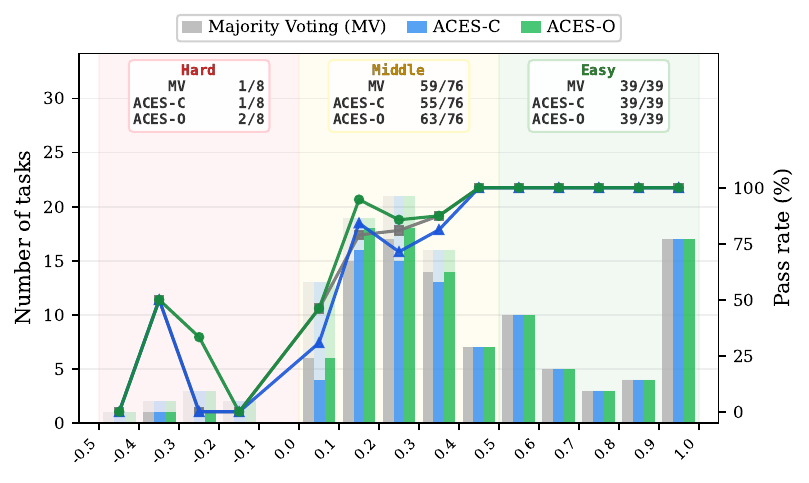}
    \caption{Pass@5}
  \end{subfigure}
  \caption{Assumption~\ref{asm:main} analysis on HumanEval$^+$~(tasks binned by $\bar{\delta}$).
  ACES-O leads in the Middle region across all $k$; the gap narrows as $k$ increases.}
  \label{fig:app-assumption-humanevalplus}
\end{figure}

\begin{figure}[H]
  \centering
  \begin{subfigure}[b]{0.33\textwidth}
    \centering
    \includegraphics[width=\textwidth]{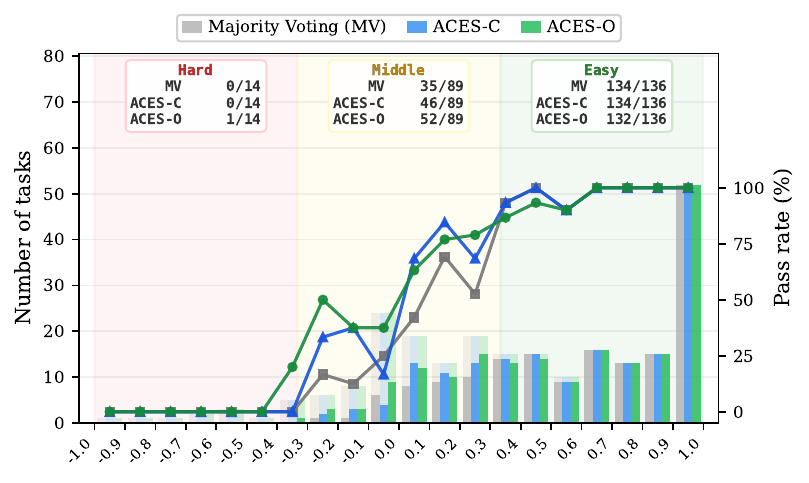}
    \caption{Pass@1}
  \end{subfigure}%
  \hfill
  \begin{subfigure}[b]{0.33\textwidth}
    \centering
    \includegraphics[width=\textwidth]{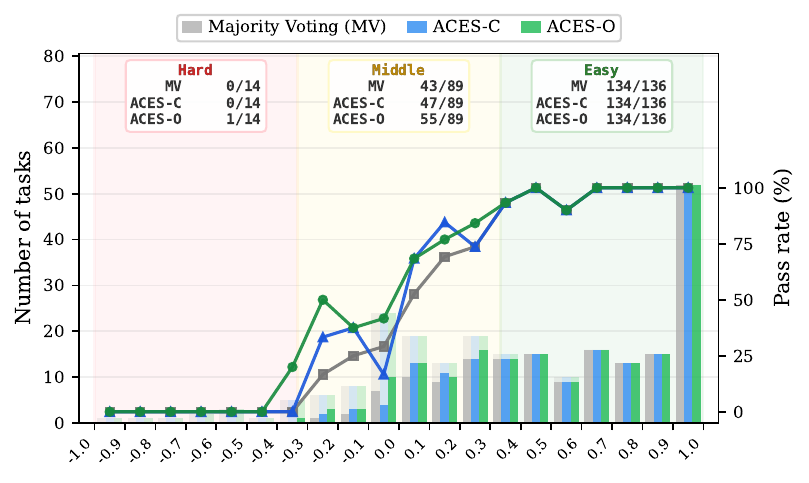}
    \caption{Pass@2}
  \end{subfigure}%
  \hfill
  \begin{subfigure}[b]{0.33\textwidth}
    \centering
    \includegraphics[width=\textwidth]{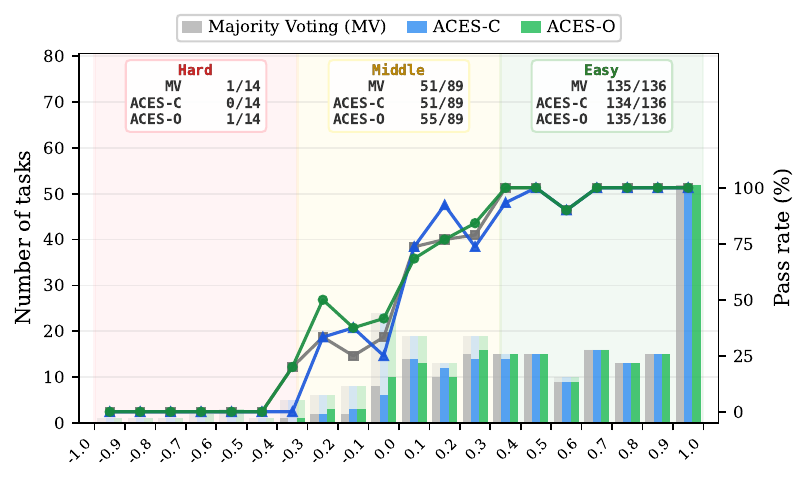}
    \caption{Pass@5}
  \end{subfigure}
  \caption{Assumption~\ref{asm:main} analysis on MBPP~(tasks binned by $\bar{\delta}$).
  ACES-O leads in the Middle region across all $k$; the gap narrows as $k$ increases.}
  \label{fig:app-assumption-mbpp}
\end{figure}

\subsection{Selection vs.\ Weighting}
\label{app:ablation-selection}

ACES-C combines two effects: \emph{test selection}~(filtering out misleading tests via the $\max(0, \cdot)$ threshold) and \emph{non-uniform weighting}~(assigning higher weight to more discriminative tests).
To disentangle their contributions, we first establish a theoretical guarantee for filtering alone, then empirically quantify the additional benefit of weighting.

\textbf{Theoretical motivation.}
By Theorem~\ref{lem:loo-identity}, $\E[\looauc_j(w)] - 1/2 = c_j(w) \cdot \delta_j$, where $c_j(w_{\mathrm{unif}}) > 0$ for all~$j$ under Assumption~\ref{asm:main}~(Proposition~\ref{thm:loo-quality}).
Since the proportionality constant is positive, the sign is preserved:
\begin{equation*}
  \mathrm{sign}\!\left(\E[\looauc_j(w_{\mathrm{unif}})] - \tfrac{1}{2}\right) \;=\; \mathrm{sign}(\delta_j).
\end{equation*}
That is, tests with $\looauc_j > 1/2$ are informative~($\delta_j > 0$) in expectation, and those with $\looauc_j < 1/2$ are misleading~($\delta_j < 0$).

If one could perfectly identify and remove all misleading tests~($\delta_j \leq 0$), assigning uniform weights to the remaining $m'$ informative tests would provably improve the signal-to-noise ratio.
Let $R(w_{\mathrm{unif}}) = m\bar{\delta}^2$ denote the SNR of Majority Voting~(Section~\ref{sec:framework}).
Uniform weighting over only the $m'$ tests with $\delta_j > 0$ gives:
\begin{equation*}
  \frac{R(w_{\mathrm{filter}})}{R(w_{\mathrm{unif}})}
  \;=\; \underbrace{\left(\frac{\textstyle\sum_{j:\,\delta_j>0} \delta_j}{\textstyle\sum_{j=1}^m \delta_j}\right)^{\!2}}_{\geq\,1} \cdot \underbrace{\frac{m}{m'}}_{\geq\,1}
  \;\geq\; 1,
\end{equation*}
since removing non-positive $\delta_j$ terms can only increase the sum, and fewer tests reduce $\sum_j w_j^2$.
By Theorem~\ref{thm:hoeffding}, a higher $R(w)$ yields a tighter Pass@$k$ bound.

The sign identity above provides exactly such a criterion from the pass matrix alone: threshold at $\looauc_j = 1/2$.
We define ACES-C Filter as:
\begin{equation*}
  w_j^{\mathrm{filter}} \;=\; \begin{cases} 1/m' & \text{if } \looauc_j(w_{\mathrm{unif}}) > 1/2, \\ 0 & \text{otherwise,} \end{cases}
  \qquad m' = |\{j : \looauc_j(w_{\mathrm{unif}}) > 1/2\}|.
\end{equation*}

\textbf{Experimental design.}
We compare three methods:
(1)~Majority Voting~(uniform weights on all tests),
(2)~ACES-C Filter~(discard tests with $\looauc_j \leq 1/2$, uniform weights on the rest), and
(3)~ACES-C~(LOO-AUC excess weighting, Eq.~\eqref{eq:aces-weight}).

\begin{table}[H]
  \centering
  \caption{Selection vs.\ weighting~(Pass@1). Subscripts denote the change from GPT-3.5-Turbo direct inference. ACES-C Filter isolates the contribution of test selection; ACES-C adds non-uniform weighting on top.}
  \label{tab:ablation-selection}
  \begin{tabular}{lccc}
    \toprule
    Method & HumanEval & HumanEval$^+$ & MBPP \\
    \midrule
    GPT-3.5-Turbo   & 68.38 & 58.75 & 66.80 \\
    \midrule
    Majority Voting  & 80.49\up{12.1} & 69.51\up{10.8} & 68.62\up{1.8} \\
    ACES-C Filter    & 81.10\up{12.7} & 69.51\up{10.8} & 69.32\up{2.5} \\
    ACES-C           & 82.93\up{14.6} & 71.34\up{12.6} & 71.19\up{4.4} \\
    \bottomrule
  \end{tabular}
\end{table}

\textbf{Results.}
Consistent with the theoretical guarantee, ACES-C Filter matches or improves over Majority Voting on all three benchmarks~(+0.61 on HumanEval, +0.00 on HumanEval$^+$, +0.70 on MBPP).
However, the gains are modest because most misleading tests have $|\delta_j|$ close to zero~(Figure~\ref{fig:loo-scatter}), so removing them changes $R(w)$ only slightly.
The additional gain from non-uniform weighting~(ACES-C vs.\ ACES-C Filter) is substantially larger~(+1.83 on HumanEval, +1.83 on HumanEval$^+$, +1.87 on MBPP), consistently accounting for over 70\% of the total improvement.
This demonstrates that the primary value of LOO-AUC lies not in binary test filtering, but in the continuous weight magnitudes that amplify informative tests proportionally to their discriminative power.

\subsection{Sensitivity to Number of Tests}
\label{app:sensitivity-m}

Figure~\ref{fig:sensitivity-m} shows Pass@1 as a function of the number of available tests~$m'$, where tests are randomly subsampled from the full test suite~(averaged over 5 trials).
All methods use the same configuration as in the main experiments~(Section~\ref{sec:exp-setup}).

All three methods improve consistently with~$m'$, but their scaling behaviors differ substantially.

MV plateaus around $m' = 50$--$100$.
Since uniform weighting treats informative and misleading tests equally, adding more tests of mixed quality provides diminishing marginal benefit once the pool is large enough.
Both ACES variants continue improving beyond this plateau: more tests provide better LOO-AUC estimates, enabling finer-grained weight differentiation between informative and misleading tests.
At $m' = 100$, ACES-C already outperforms MV at full~$m' = 500$ on all benchmarks, showing that intelligent weighting is more valuable than simply collecting more tests.

The gap between ACES-O and ACES-C generally widens with~$m'$: at small~$m'$ the two variants perform similarly, but with more tests ACES-O's optimization can exploit the richer signal to further refine the weights.
Practically, even a modest test budget~($m' \approx 100$) suffices for the LOO-AUC weighting to capture most of the available signal.

\begin{figure}[H]
  \centering
  \includegraphics[width=\textwidth]{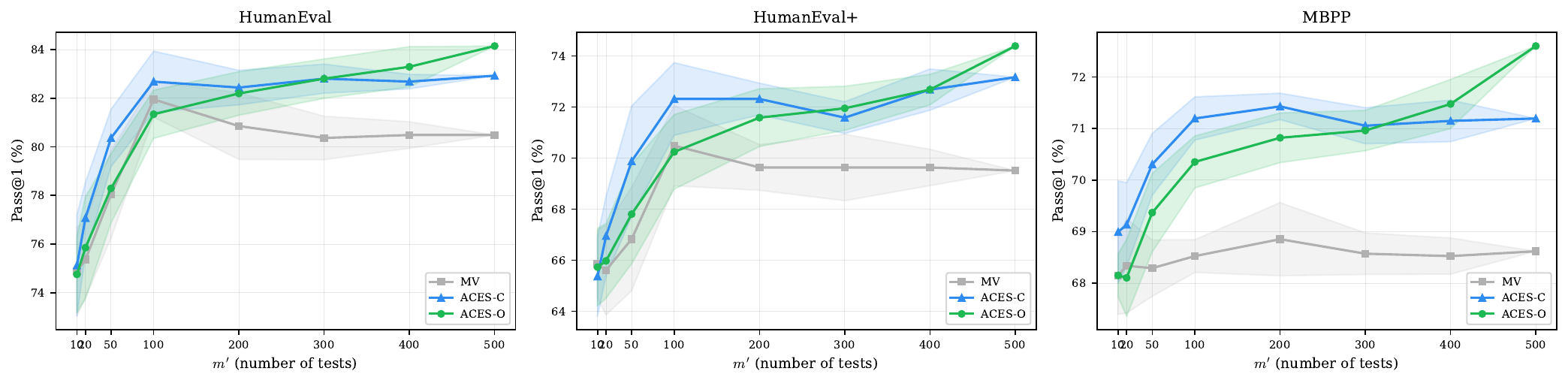}
  \caption{Sensitivity to the number of tests~$m'$. Tests are randomly subsampled; results averaged over 5 trials~(shaded: $\pm$ std). MV plateaus around $m'{=}50$--$100$; both ACES variants continue improving and surpass MV at $m'{=}500$ with only $m'{=}100$ tests.}
  \label{fig:sensitivity-m}
\end{figure}

\subsection{Sensitivity to Number of Candidate Codes}
\label{app:sensitivity-n}

Figure~\ref{fig:sensitivity-n} shows Pass@1 as a function of the number of candidate codes~$n'$, where codes are stratified-subsampled to preserve the correct/incorrect ratio~$\pi$~(averaged over 10 trials).
All $m$ tests are retained; all methods use the same configuration as in the main experiments.

All three methods are stable with respect to~$n'$~(standard deviations $\leq 1.6\%$ throughout), but their scaling behaviors differ.

MV is approximately flat across all~$n'$, and on MBPP it \emph{decreases} from 70.8\%~($n' = 20$) to 68.6\%~($n' = 200$): with more candidates~(most of which are incorrect), there are more incorrect codes that may outscore the best correct one under equal weighting.
ACES-C maintains performance close to its full-data result even at $n' = 20$.
This is because ACES-C computes weights in closed form from the pass matrix; even a small candidate pool suffices to produce meaningful LOO-AUC estimates.
ACES-O benefits the most from additional candidates.
At small~$n'$, the optimization has too few pairwise comparisons between passing and failing codes to learn reliable weights; as~$n'$ grows, the gradient signal becomes richer and ACES-O achieves the best performance at $n' = 200$ on all benchmarks.

\begin{figure}[H]
  \centering
  \includegraphics[width=\textwidth]{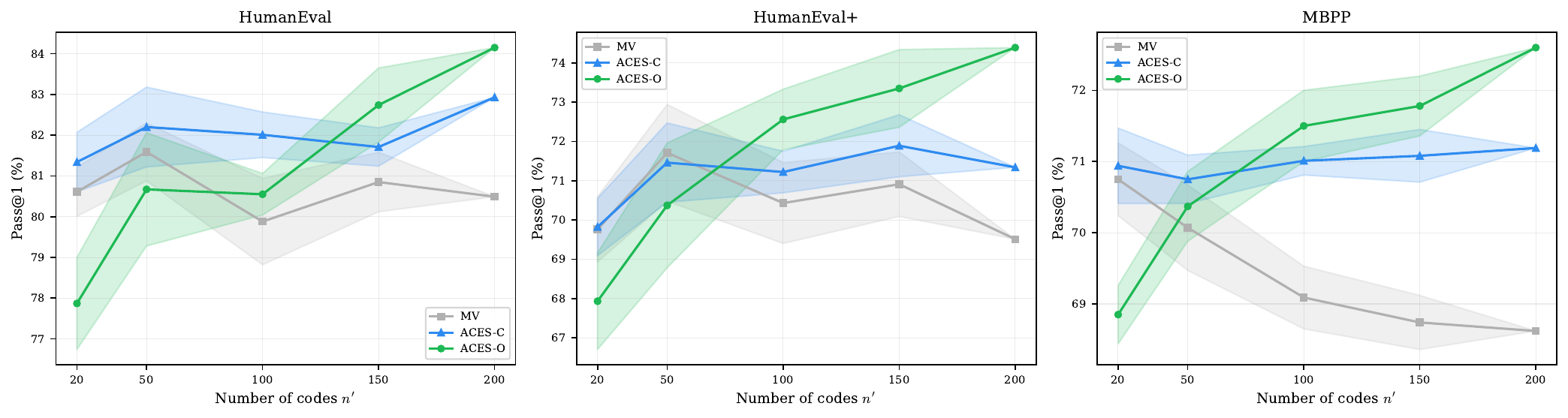}
  \caption{Sensitivity to the number of candidate codes~$n'$. Codes are stratified-subsampled to preserve the correct/incorrect ratio; results averaged over 10 trials~(shaded: $\pm$ std). MV is flat or declining; ACES-C is robust even at small~$n'$; ACES-O improves steadily and leads at $n'{=}200$.}
  \label{fig:sensitivity-n}
\end{figure}

\textbf{Summary of data sensitivity~(Appendices~\ref{app:sensitivity-m}--\ref{app:sensitivity-n}).}
The two sensitivity analyses reveal complementary strengths of the three methods.
MV plateaus early with~$m'$ and can even degrade with larger~$n'$, since uniform weighting cannot exploit additional tests or candidates.
ACES-C is robust to both~$m'$ and~$n'$, requiring neither many tests nor many candidates to achieve strong performance; it is the safest choice under limited data budgets.
ACES-O achieves the best performance when both~$m'$ and~$n'$ are sufficiently large, as the optimization benefits from richer pairwise signal; in the standard setting~($n {=} 200$, $m \approx 500$), it consistently leads on all benchmarks.

\subsection{Effect of Pre-Filtering Cutoff}
\label{app:sensitivity-K}

Figure~\ref{fig:sensitivity-K} shows the effect of the pre-filtering cutoff~$K$ on both ACES variants.
For each~$K$, we select the top-$K$ candidates by majority vote, run the reranking method on the resulting $K \times m$ submatrix, and score all $n = 200$ codes with the learned weights.

The two variants exhibit complementary trends, each excelling in a different regime.
ACES-C generally improves as~$K$ increases, reaching its best performance at $K = 200$~(no pre-filtering) on all benchmarks.
This is consistent with the findings in Appendix~\ref{app:sensitivity-n}: more codes provide better LOO-AUC estimates for the closed-form weighting.

ACES-O achieves its best results at small to moderate~$K$~(8--32), where the pre-filtered candidate pool concentrates on high-quality codes and provides a cleaner optimization landscape.
Performance is strong across a wide range of~$K$ values and declines gradually for larger~$K$, since including many low-quality candidates introduces noise into the pairwise comparisons.

These complementary profiles are consistent with the design of the two variants: ACES-C operates on all~$n$ candidates to maximize its closed-form estimates, while ACES-O benefits from a focused candidate pool for optimization.

\begin{figure}[H]
  \centering
  \includegraphics[width=\textwidth]{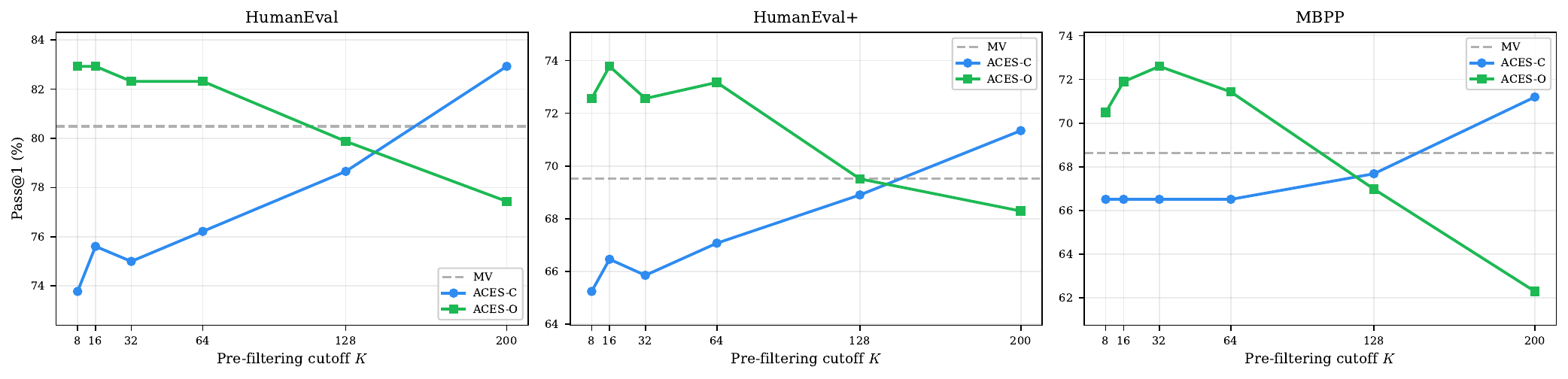}
  \caption{Sensitivity to the pre-filtering cutoff~$K$. ACES-C~(blue) benefits from larger candidate pools; ACES-O~(green) benefits from focused pre-filtering. Dashed line: MV baseline~(independent of~$K$).}
  \label{fig:sensitivity-K}
\end{figure}

\subsection{ACES-O Convergence}
\label{app:convergence}

Figure~\ref{fig:convergence} shows the convergence behavior of ACES-O over $T = 300$ iterations, averaged across all tasks per benchmark.
The top row plots the surrogate objective $J(w) = \sum_j w_j \bigl(\widehat{\looauc}_j(w) - \tfrac{1}{2}\bigr)$~(Eq.~\ref{eq:objective}), where $\widehat{\looauc}_j$ is the logistic surrogate~(Eq.~\ref{eq:logistic-auc}).
The bottom row plots the corresponding Pass@1 (smoothed with a 15-step moving average; raw values shown in light green).

The surrogate objective $J(w)$ increases smoothly and monotonically on all three benchmarks, with the steepest gains in the first 50 iterations and near-complete convergence by iteration 80--100.
Pass@1 reflects this stability: once the surrogate converges, the discrete Pass@1 metric remains essentially flat, fluctuating by at most 0.6 percentage points between $T = 50$ and $T = 300$ on all benchmarks.
ACES-O surpasses both MV and ACES-C early in the optimization and maintains this advantage throughout, confirming that the objective $J(w)$ is well-aligned with downstream Pass@$k$ performance.
The results are robust to the choice of~$T$: any value in the range $[50, 300]$ yields comparable performance.

\begin{figure}[H]
  \centering
  \includegraphics[width=\textwidth]{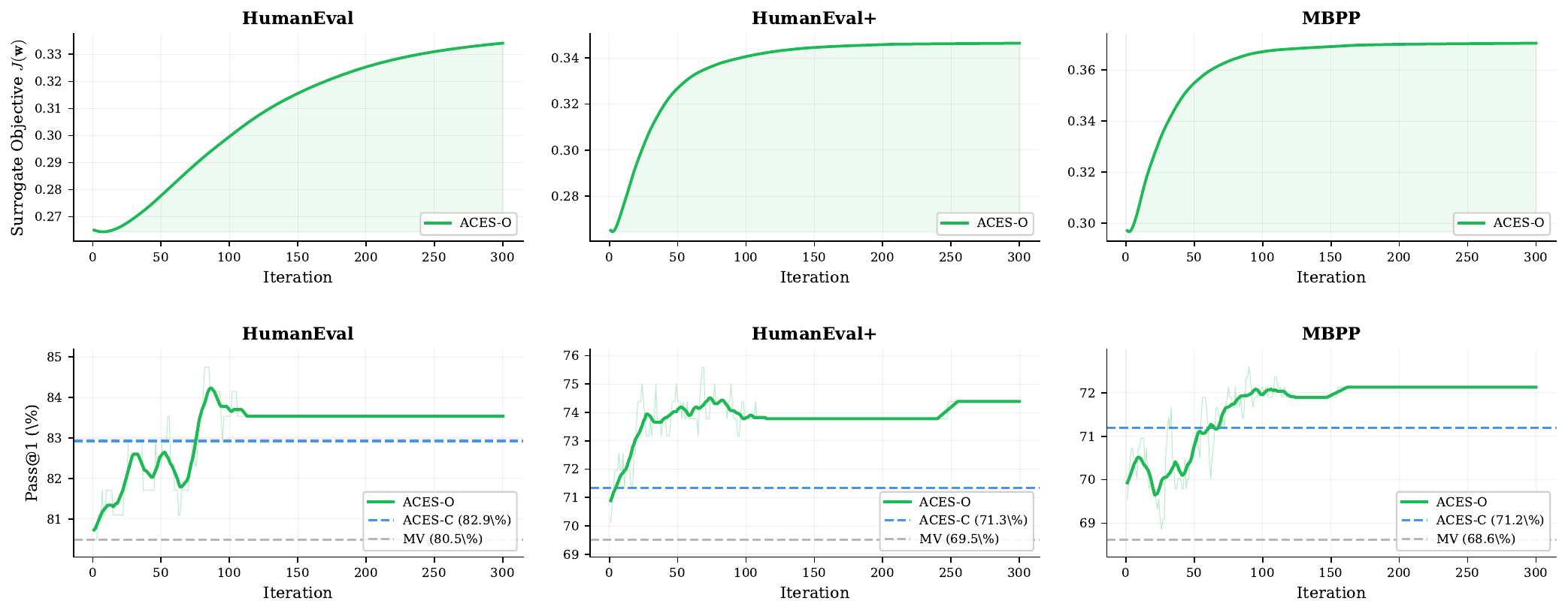}
  \caption{ACES-O convergence over 300 iterations, averaged across all tasks. \textbf{Top:} surrogate objective~$J(\mathbf{w})$ converges by iteration 80--100. \textbf{Bottom:} Pass@1~(smoothed; raw in light green) surpasses both MV~(dashed gray) and ACES-C~(dashed blue) and remains stable for $T \in [50, 300]$.}
  \label{fig:convergence}
\end{figure}

\subsection{ACES-O Hyperparameter Sensitivity}
\label{app:sensitivity-hp}

Figure~\ref{fig:sensitivity-hp} shows the effect of varying the two key ACES-O hyperparameters independently: the logistic surrogate sharpness~$\gamma$~(top row) and the learning rate~$\eta$~(bottom row).
When sweeping one parameter, the other is held at its default value.

Both hyperparameters exhibit broad stable regions.
For~$\gamma$, the only requirement is that the logistic surrogate be sharp enough to approximate the indicator function: $\gamma \leq 2$ is insufficient, but any $\gamma \geq 5$ yields strong performance, with the entire range $[5, 50]$ varying by fewer than 2 percentage points on all benchmarks.
For~$\eta$, the stable region is even wider: $\eta \in [0.005, 0.5]$ consistently outperforms MV, again varying by at most 2 percentage points.
The values used in the main experiments fall well within these stable regions.

\begin{figure}[H]
  \centering
  \includegraphics[width=\textwidth]{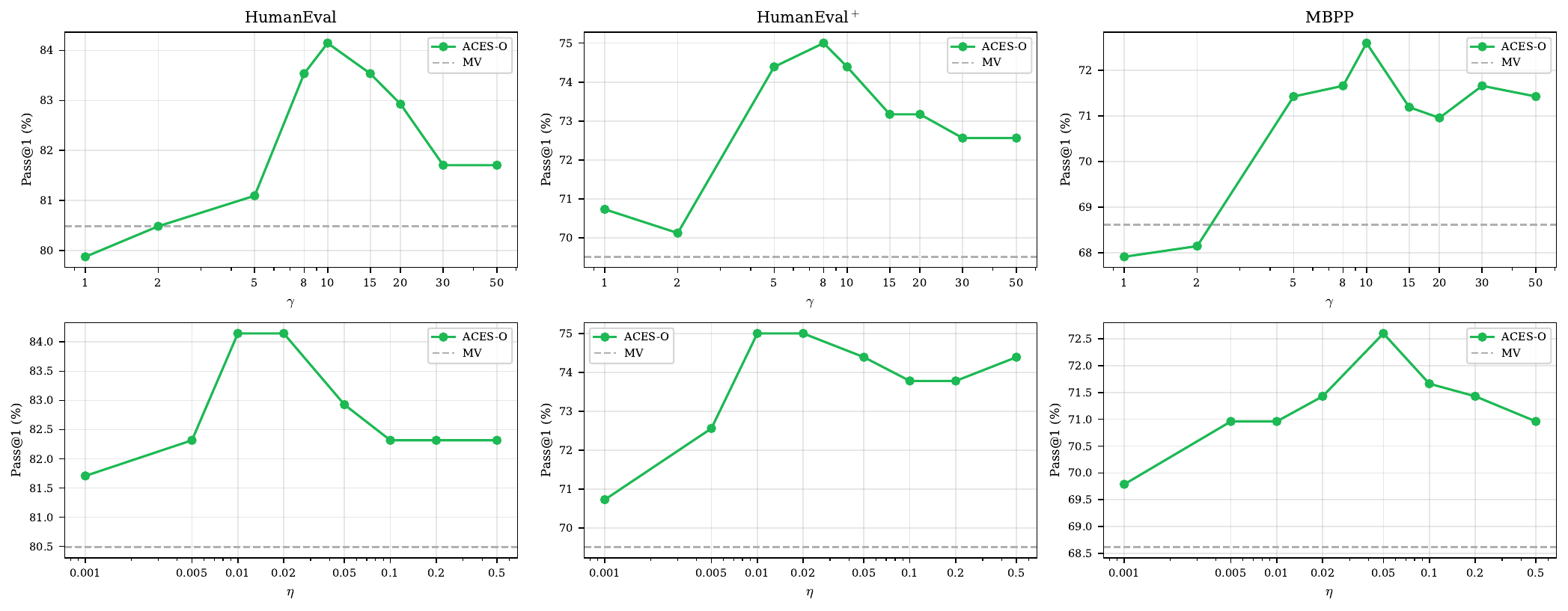}
  \caption{ACES-O hyperparameter sensitivity. \textbf{Top:} Pass@1 vs.\ surrogate sharpness~$\gamma$~(other parameters at defaults). \textbf{Bottom:} Pass@1 vs.\ learning rate~$\eta$. Dashed gray line: MV baseline. For $\gamma \geq 5$ and $\eta \in [0.005, 0.5]$, ACES-O consistently outperforms MV on all benchmarks.}
  \label{fig:sensitivity-hp}
\end{figure}

\subsection{Computational Cost}
\label{app:runtime}

Figure~\ref{fig:runtime} compares the reranking time per task for five methods~($n = 200$, $m = 500$), measured on a single CPU core of an Apple M3 MacBook Air~(16\,GB RAM) and averaged over all tasks.
We time only the reranking step; shared upstream costs~(code generation, test generation, code execution) are excluded.
To ensure a fair comparison, all execution outputs are pre-hashed before timing, so that string comparisons in output-level methods~(MBR-exec) reduce to integer comparisons.

All methods finish within one second per task.
MV, CodeT, ACES-C, and ACES-O operate solely on the binary pass matrix; among these, ACES-C~(9\,ms) adds negligible overhead over MV~(4\,ms) thanks to its closed-form computation.
ACES-O requires ${\sim}0.85$\,s for iterative optimization, comparable to MBR-exec~(${\sim}0.38$\,s), which incurs $O(n^2 m)$ pairwise output comparisons.
In practice, the reranking cost for all methods is negligible relative to upstream code execution, which dominates wall-clock time.
Note that all methods compared above operate solely on execution results~(the binary pass matrix or output strings).
Methods that incorporate static analysis or invoke LLMs for repeated evaluation incur substantially higher costs, often orders of magnitude greater, due to the expense of additional model inference, and are therefore not included in this timing comparison.

\begin{figure}[H]
  \centering
  \includegraphics[width=0.55\textwidth]{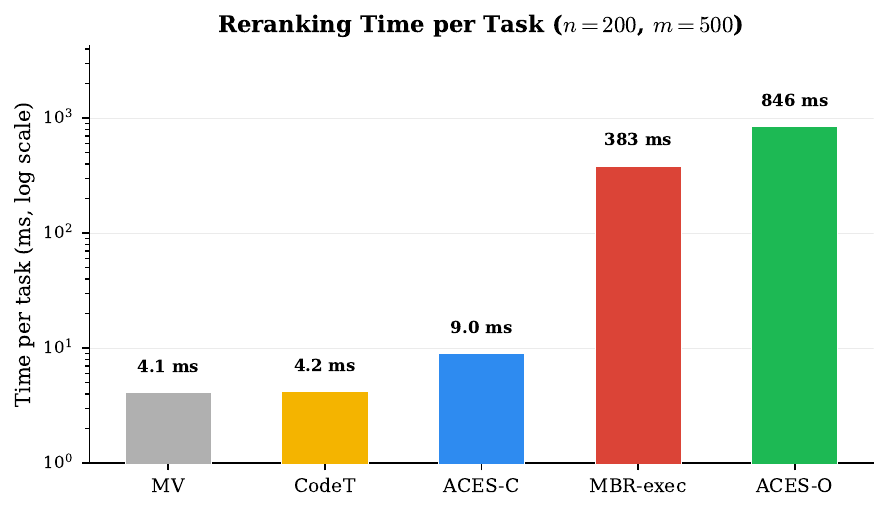}
  \caption{Reranking time per task~(log scale, $n{=}200$, $m{=}500$, single CPU core). Shared upstream costs are excluded; execution outputs are pre-hashed for fair comparison. All methods finish within one second.}
  \label{fig:runtime}
\end{figure}

\subsection{Pairwise Vote Statistics}
\label{app:noise-stats}

To complement the assumption analysis, we directly measure the empirical distribution of the three pairwise vote types defined in Table~\ref{tab:noise}.
For each non-trivial task, we enumerate all $(c^+, c^-, t_j)$ triples and classify each vote $h_j = B_{c^+,j} - B_{c^-,j}$ as informative~($+1$), uninformative~($0$), or misleading~($-1$).
Table~\ref{tab:vote-stats} reports the aggregate proportions.

\begin{table}[H]
  \centering
  \caption{Pairwise vote distribution across benchmarks. Informative votes consistently dominate misleading ones, confirming Assumption~\ref{asm:main}.}
  \label{tab:vote-stats}
  \begin{tabular}{lcccc}
    \toprule
    Benchmark & Informative (\%) & Uninformative (\%) & Misleading (\%) & Inf/Mis ratio \\
    \midrule
    HumanEval      & 22.6 & 72.2 & 5.3 & 4.3$\times$ \\
    HumanEval$^+$  & 20.7 & 74.9 & 4.4 & 4.7$\times$ \\
    MBPP           & 30.7 & 59.1 & 10.2 & 3.0$\times$ \\
    \bottomrule
  \end{tabular}
\end{table}

Misleading votes are rare across all benchmarks~($\leq 10.2\%$), confirming that the majority of pairwise test--code interactions provide either correct or neutral ranking signal.
Uninformative votes dominate~(59--75\%): even after removing constant-column tests, most non-constant tests agree on any given $(c^+, c^-)$ pair, since a test that distinguishes some pairs may still assign the same outcome to others.
This highlights the value of concentrating weight on the subset of tests that actively discriminate.
MBPP has roughly twice the misleading rate of HumanEval~(10.2\% vs.\ 5.3\%), consistent with its lower assumption satisfaction rate~(71.1\% vs.\ 83.1\%; Table~\ref{tab:assumption-summary}).
The informative-to-misleading ratio remains $\geq 3\times$ on all benchmarks, providing a comfortable margin for Assumption~\ref{asm:main}.

\subsection{Generation Prompts}
\label{app:prompts}

\textbf{HumanEval, HumanEval$^+$, and MBPP.}
We use the prompts from the MPSC protocol~\citep{huang2024enhancing}, reproduced below:

\begin{promptbox}[Prompt for generating candidate solutions (HumanEval / HumanEval\textsuperscript{+} / MBPP)]
\begin{lstlisting}[style=prompt]
I want you to act like a Python programmer. I will give you the declaration of a function and comments about its property. You need to implement the body of the function in the code block. Do not modify any code I provide. Do not provide any explanations.

Here is the question.
```Python
<@\textcolor{blue}{\{Docstring\}}@>
```
\end{lstlisting}
\end{promptbox}

\begin{promptbox}[Prompt for generating test cases (HumanEval / HumanEval\textsuperscript{+} / MBPP)]
\begin{lstlisting}[style=prompt]
```Python
# Given a docstring, continue to write the
# following code with 10 valid assertion
# statements to check the correctness of the
# function. Provide diverse test cases.
<@\textcolor{blue}{\{Docstring\}}@>
    pass

# check the correctness of with 10 different
# valid assertion statements in the form of
# "assert <@\textcolor{magenta}{\{entry point\}}@>(...) == ..."
assert
\end{lstlisting}
\end{promptbox}

\textbf{LeetCodeDataset~(Appendix~\ref{app:additional-models}).}

\begin{promptbox}[Prompt for generating candidate solutions (LeetCodeDataset)]
\begin{lstlisting}[style=prompt]
You are an expert Python programmer. You will be given a question (problem specification) and will generate a correct Python program that matches the specification and passes all tests.

### Question:
<@\textcolor{blue}{\{question\}}@>

### Format: You will use the following starter code to write the solution to the problem and enclose your code within delimiters.
```python
<@\textcolor{blue}{\{starter\_code\}}@>
```

### Answer: (use the provided format with backticks)
\end{lstlisting}
\end{promptbox}

\begin{promptbox}[Prompt for generating test cases (LeetCodeDataset)]
\begin{lstlisting}[style=prompt]
You are an expert Python programmer. You will be given a question (problem specification) and starter code, and you will generate diverse and semantically correct Python assertion statements to test the target function.

### Question:
<@\textcolor{blue}{\{question\}}@>

### Starter Code:
```python
<@\textcolor{blue}{\{starter\_code\}}@>
```

### Target Call:
<@\textcolor{blue}{\{entry\_point\}}@>

### Format: Output exactly 10 raw Python assertion statements. Each line must start with `assert `. Use the exact target call that is specified. Do not output markdown, code fences, comments, explanations, helper functions, or any text other than the assertions.

### Answer:
\end{lstlisting}
\end{promptbox}
 

\begin{thebibliography}{52}
\providecommand{\natexlab}[1]{#1}
\providecommand{\url}[1]{\texttt{#1}}
\expandafter\ifx\csname urlstyle\endcsname\relax
  \providecommand{\doi}[1]{doi: #1}\else
  \providecommand{\doi}{doi: \begingroup \urlstyle{rm}\Url}\fi

\bibitem[Achiam et~al.(2023)Achiam, Adler, Agarwal, Ahmad, Akkaya, Aleman,
  Almeida, Altenschmidt, Altman, Anadkat, et~al.]{achiam2023gpt}
Josh Achiam, Steven Adler, Sandhini Agarwal, Lama Ahmad, Ilge Akkaya,
  Florencia~Leoni Aleman, Diogo Almeida, Janko Altenschmidt, Sam Altman,
  Shyamal Anadkat, et~al.
\newblock {GPT-4} technical report.
\newblock \emph{Arxiv Preprint Arxiv:2303.08774}, 2023.

\bibitem[Agarwal et~al.(2005)Agarwal, Graepel, Herbrich, Har-Peled, and
  Roth]{agarwal2005generalization}
Shivani Agarwal, Thore Graepel, Ralf Herbrich, Sariel Har-Peled, and Dan Roth.
\newblock Generalization bounds for the area under the roc curve.
\newblock \emph{Journal of Machine Learning Research}, 6\penalty0
  (14):\penalty0 393--425, 2005.

\bibitem[Austin et~al.(2021)Austin, Odena, Nye, Bosma, Michalewski, Dohan,
  Jiang, Cai, Terry, Le, et~al.]{austin2021program}
Jacob Austin, Augustus Odena, Maxwell Nye, Maarten Bosma, Henryk Michalewski,
  David Dohan, Ellen Jiang, Carrie Cai, Michael Terry, Quoc Le, et~al.
\newblock Program synthesis with large language models.
\newblock \emph{Arxiv Preprint Arxiv:2108.07732}, 2021.

\bibitem[Brown et~al.(2024)Brown, Juravsky, Ehrlich, Clark, Le, R{\'e}, and
  Mirhoseini]{brown2024large}
Bradley Brown, Jordan Juravsky, Ryan Ehrlich, Ronald Clark, Quoc~V Le,
  Christopher R{\'e}, and Azalia Mirhoseini.
\newblock Large language monkeys: Scaling inference compute with repeated
  sampling.
\newblock \emph{Arxiv Preprint Arxiv:2407.21787}, 2024.

\bibitem[Chen et~al.(2023)Chen, Zhang, Nguyen, Zan, Lin, Lou, and
  Chen]{chen2022codet}
Bei Chen, Fengji Zhang, Anh Nguyen, Daoguang Zan, Zeqi Lin, Jian-Guang Lou, and
  Weizhu Chen.
\newblock {CodeT}: Code generation with generated tests.
\newblock In \emph{The Eleventh International Conference on Learning
  Representations}, 2023.

\bibitem[Chen et~al.(2021)Chen, Tworek, Jun, Yuan, Pinto, Kaplan, Edwards,
  Burda, Joseph, Brockman, et~al.]{chen2021evaluating}
Mark Chen, Jerry Tworek, Heewoo Jun, Qiming Yuan, Henrique Ponde De~Oliveira
  Pinto, Jared Kaplan, Harri Edwards, Yuri Burda, Nicholas Joseph, Greg
  Brockman, et~al.
\newblock Evaluating large language models trained on code.
\newblock \emph{Arxiv Preprint Arxiv:2107.03374}, 2021.

\bibitem[Chen et~al.(2024)Chen, Lin, Sch{\"a}rli, and Zhou]{chen2023teaching}
Xinyun Chen, Maxwell Lin, Nathanael Sch{\"a}rli, and Denny Zhou.
\newblock Teaching large language models to self-debug.
\newblock In \emph{The Twelfth International Conference on Learning
  Representations}, 2024.

\bibitem[Cl{\'e}men{\c{c}}on et~al.(2008)Cl{\'e}men{\c{c}}on, Lugosi, and
  Vayatis]{clemenson2008ranking}
St{\'e}phan Cl{\'e}men{\c{c}}on, G{\'a}bor Lugosi, and Nicolas Vayatis.
\newblock Ranking and empirical minimization of u-statistics.
\newblock \emph{The Annals of Statistics}, pp.\  844--874, 2008.

\bibitem[Cobbe et~al.(2021)Cobbe, Kosaraju, Bavarian, Chen, Jun, Kaiser,
  Plappert, Tworek, Hilton, Nakano, et~al.]{cobbe2021training}
Karl Cobbe, Vineet Kosaraju, Mohammad Bavarian, Mark Chen, Heewoo Jun, Lukasz
  Kaiser, Matthias Plappert, Jerry Tworek, Jacob Hilton, Reiichiro Nakano,
  et~al.
\newblock Training verifiers to solve math word problems, 2021.
\newblock \emph{Arxiv Preprint Arxiv:2110.14168}, 2021.

\bibitem[Dawid \& Skene(1979)Dawid and Skene]{dawid1979maximum}
Alexander~Philip Dawid and Allan~M Skene.
\newblock Maximum likelihood estimation of observer error-rates using the em
  algorithm.
\newblock \emph{Journal of the Royal Statistical Society: Series C (applied
  Statistics)}, 28\penalty0 (1):\penalty0 20--28, 1979.

\bibitem[Dong et~al.(2024)Dong, Jiang, Jin, and Li]{dong2024self}
Yihong Dong, Xue Jiang, Zhi Jin, and Ge~Li.
\newblock Self-collaboration code generation via {ChatGPT}.
\newblock \emph{Acm Transactions on Software Engineering and Methodology},
  33\penalty0 (7), 2024.

\bibitem[Du et~al.(2025)Du, Sun, and Li]{cast_tse}
Yali Du, Hui Sun, and Ming Li.
\newblock Post-incorporating code structural knowledge into pretrained models
  via icl for code translation.
\newblock \emph{IEEE Transactions on Software Engineering}, 51\penalty0
  (11):\penalty0 3038--3055, 2025.

\bibitem[Eikema \& Aziz(2022)Eikema and Aziz]{eikema2022sampling}
Bryan Eikema and Wilker Aziz.
\newblock Sampling-based approximations to minimum {B}ayes risk decoding for
  neural machine translation.
\newblock In \emph{Proceedings of the 2022 Conference on Empirical Methods in
  Natural Language Processing}, pp.\  10978--10993. Association for
  Computational Linguistics, 2022.

\bibitem[Freund et~al.(2003)Freund, Iyer, Schapire, and
  Singer]{freund2003efficient}
Yoav Freund, Raj Iyer, Robert~E Schapire, and Yoram Singer.
\newblock An efficient boosting algorithm for combining preferences.
\newblock \emph{Journal of Machine Learning Research}, 4\penalty0
  (Nov):\penalty0 933--969, 2003.

\bibitem[Gao \& Zhou(2015)Gao and Zhou]{gao2015auc}
Wei Gao and Zhi-Hua Zhou.
\newblock On the consistency of {AUC} pairwise optimization.
\newblock In \emph{International Joint Conference on Artificial Intelligence},
  pp.\  939--945, 2015.

\bibitem[Guo et~al.(2024)Guo, Zhu, Yang, Xie, Dong, Zhang, Chen, Bi, Wu, Li,
  et~al.]{guo2024deepseek}
Daya Guo, Qihao Zhu, Dejian Yang, Zhenda Xie, Kai Dong, Wentao Zhang, Guanting
  Chen, Xiao Bi, Yifan Wu, YK~Li, et~al.
\newblock {DeepSeek-Coder}: When the large language model meets
  programming--the rise of code intelligence.
\newblock \emph{Arxiv Preprint Arxiv:2401.14196}, 2024.

\bibitem[Haddad(2022)]{haddad2022noise}
Dany Haddad.
\newblock Noise tolerance of learning to rank under class-conditional label
  noise.
\newblock \emph{Arxiv Preprint Arxiv:2208.02126}, 2022.

\bibitem[Hoeffding(1963)]{hoeffding1963probability}
Wassily Hoeffding.
\newblock Probability inequalities for sums of bounded random variables.
\newblock \emph{Journal of the American Statistical Association}, 58\penalty0
  (301):\penalty0 13--30, 1963.

\bibitem[Huang et~al.(2024)Huang, Lu, Wan, and Duan]{huang2024enhancing}
Baizhou Huang, Shuai Lu, Xiaojun Wan, and Nan Duan.
\newblock Enhancing large language models in coding through multi-perspective
  self-consistency.
\newblock In \emph{Proceedings of the 62nd Annual Meeting of the Association
  for Computational Linguistics (volume 1: Long Papers)}, pp.\  1429--1450.
  Association for Computational Linguistics, 2024.

\bibitem[Hui et~al.(2024)Hui, Yang, Cui, Yang, Liu, Zhang, Liu, Zhang, Yu, Lu,
  et~al.]{hui2024qwen2}
Binyuan Hui, Jian Yang, Zeyu Cui, Jiaxi Yang, Dayiheng Liu, Lei Zhang, Tianyu
  Liu, Jiajun Zhang, Bowen Yu, Keming Lu, et~al.
\newblock {Qwen2. 5-Coder} technical report.
\newblock \emph{Arxiv Preprint Arxiv:2409.12186}, 2024.

\bibitem[Li et~al.(2025{\natexlab{a}})Li, Cao, Cao, Li, Tan, Keutzer, Xing,
  Gonzalez, and Stoica]{li-etal-2025-test}
Dacheng Li, Shiyi Cao, Chengkun Cao, Xiuyu Li, Shangyin Tan, Kurt Keutzer,
  Jiarong Xing, Joseph~E. Gonzalez, and Ion Stoica.
\newblock {S}*: Test time scaling for code generation.
\newblock In \emph{Findings of the Association for Computational Linguistics:
  Emnlp 2025}, pp.\  15964--15978. Association for Computational Linguistics,
  2025{\natexlab{a}}.

\bibitem[Li et~al.(2024)Li, Fernandes, Gurevych, and Martins]{li2024doce}
Haau-Sing Li, Patrick Fernandes, Iryna Gurevych, and Andr{\'e}~FT Martins.
\newblock {DOCE}: Finding the sweet spot for execution-based code generation.
\newblock \emph{Arxiv Preprint Arxiv:2408.13745}, 2024.

\bibitem[Li et~al.(2025{\natexlab{b}})Li, Yuan, Yu, Guo, and Cao]{11098743}
Kefan Li, Yuan Yuan, Hongyue Yu, Tingyu Guo, and Shijie Cao.
\newblock {CoCoEvo}: Co-evolution of programs and test cases to enhance code
  generation.
\newblock \emph{Ieee Transactions on Evolutionary Computation}, pp.\  1--1,
  2025{\natexlab{b}}.

\bibitem[Li et~al.(2022)Li, Choi, Chung, Kushman, Schrittwieser, Leblond,
  Eccles, Keeling, Gimeno, Dal~Lago, et~al.]{li2022alphacode}
Yujia Li, David Choi, Junyoung Chung, Nate Kushman, Julian Schrittwieser,
  R{\'e}mi Leblond, Tom Eccles, James Keeling, Felix Gimeno, Agustin Dal~Lago,
  et~al.
\newblock Competition-level code generation with alphacode.
\newblock \emph{Science}, 378\penalty0 (6624):\penalty0 1092--1097, 2022.

\bibitem[Lightman et~al.(2024)Lightman, Kosaraju, Burda, Edwards, Baker, Lee,
  Leike, Schulman, Sutskever, and Cobbe]{lightman2023let}
Hunter Lightman, Vineet Kosaraju, Yuri Burda, Harrison Edwards, Bowen Baker,
  Teddy Lee, Jan Leike, John Schulman, Ilya Sutskever, and Karl Cobbe.
\newblock Let's verify step by step.
\newblock In \emph{The Twelfth International Conference on Learning
  Representations}, 2024.

\bibitem[Liu et~al.(2023)Liu, Xia, Wang, and ZHANG]{liu2023your}
Jiawei Liu, Chunqiu~Steven Xia, Yuyao Wang, and LINGMING ZHANG.
\newblock Is your code generated by chat{GPT} really correct? rigorous
  evaluation of large language models for code generation.
\newblock In \emph{Thirty-Seventh Conference on Neural Information Processing
  Systems}, 2023.

\bibitem[Liu et~al.(2025)Liu, Li, Yang, Sun, and Li]{pmlr-v267-liu25ah}
Ren-Biao Liu, Anqi Li, Chaoding Yang, Hui Sun, and Ming Li.
\newblock Revisiting {Chain-of-Thought} in code generation: Do language models
  need to learn reasoning before coding?
\newblock In \emph{Proceedings of the 42nd International Conference on Machine
  Learning}, volume 267, pp.\  38809--38826. PMLR, 13--19 Jul 2025.

\bibitem[Liu et~al.(2026)Liu, Xue, Ma, Sun, Li, and
  Li]{Liu_Xue_Ma_Sun_Li_Li_2026}
Ren-Biao Liu, Jiang-Tian Xue, Chao-Zeng Ma, Hui Sun, Xin-Ye Li, and Ming Li.
\newblock {Dynamic-Static} synergistic selection method for candidate code
  solutions with generated test cases.
\newblock In \emph{Proceedings of the Aaai Conference on Artificial
  Intelligence}, pp.\  32096--32104, 2026.

\bibitem[Liu(2010)]{liu2009learning}
Tie-Yan Liu.
\newblock Learning to rank for information retrieval.
\newblock In \emph{Proceedings of the 33rd International Acm Sigir Conference
  on Research and Development in Information Retrieval}, pp.\  904. Association
  for Computing Machinery, 2010.

\bibitem[Lord \& Novick(2008)Lord and Novick]{lord2008statistical}
Frederic~M Lord and Melvin~R Novick.
\newblock \emph{Statistical Theories of Mental Test Scores}.
\newblock IAP, 2008.

\bibitem[Lozhkov et~al.(2024)Lozhkov, Li, Allal, Cassano, Lamy-Poirier, Tazi,
  Tang, Pykhtar, Liu, Wei, et~al.]{lozhkov2024starcoder}
Anton Lozhkov, Raymond Li, Loubna~Ben Allal, Federico Cassano, Joel
  Lamy-Poirier, Nouamane Tazi, Ao~Tang, Dmytro Pykhtar, Jiawei Liu, Yuxiang
  Wei, et~al.
\newblock Starcoder 2 and the stack v2: The next generation.
\newblock \emph{Arxiv Preprint Arxiv:2402.19173}, 2024.

\bibitem[Lukasik et~al.(2025)Lukasik, Chen, Narasimhan, Menon, Jitkrittum, Yu,
  J.~Reddi, Fu, Bateni, and Kumar]{pmlr-v267-lukasik25a}
Michal Lukasik, Lin Chen, Harikrishna Narasimhan, Aditya~Krishna Menon,
  Wittawat Jitkrittum, Felix~X. Yu, Sashank J.~Reddi, Gang Fu, Mohammadhossein
  Bateni, and Sanjiv Kumar.
\newblock Bipartite ranking from multiple labels: On loss versus label
  aggregation.
\newblock In \emph{Proceedings of the 42nd International Conference on Machine
  Learning}, volume 267, pp.\  41074--41102. PMLR, 13--19 Jul 2025.

\bibitem[Luo et~al.(2024)Luo, Xu, Zhao, Sun, Geng, Hu, Tao, Ma, Lin, and
  Jiang]{luo2024wizardcoder}
Ziyang Luo, Can Xu, Pu~Zhao, Qingfeng Sun, Xiubo Geng, Wenxiang Hu, Chongyang
  Tao, Jing Ma, Qingwei Lin, and Daxin Jiang.
\newblock {WizardCoder}: Empowering code large language models with
  evol-instruct.
\newblock In \emph{The Twelfth International Conference on Learning
  Representations}, 2024.

\bibitem[Nguyen et~al.(2024)Nguyen, Ibrahim, and Fu]{nguyen2024noisy}
Tri Nguyen, Shahana Ibrahim, and Xiao Fu.
\newblock Noisy label learning with instance-dependent outliers:
  Identifiability via crowd wisdom.
\newblock In \emph{Advances in Neural Information Processing Systems}, 2024.

\bibitem[Ni et~al.(2023)Ni, Iyer, Radev, Stoyanov, Yih, Wang, and
  Lin]{ni2023lever}
Ansong Ni, Srini Iyer, Dragomir Radev, Veselin Stoyanov, Wen-tau Yih, Sida
  Wang, and Xi~Victoria Lin.
\newblock {LEVER}: Learning to verify language-to-code generation with
  execution.
\newblock In \emph{Proceedings of the 40th International Conference on Machine
  Learning}, pp.\  26106--26128. PMLR, 2023.

\bibitem[Roziere et~al.(2023)Roziere, Gehring, Gloeckle, Sootla, Gat, Tan, Adi,
  Liu, Sauvestre, Remez, et~al.]{roziere2023code}
Baptiste Roziere, Jonas Gehring, Fabian Gloeckle, Sten Sootla, Itai Gat,
  Xiaoqing~Ellen Tan, Yossi Adi, Jingyu Liu, Romain Sauvestre, Tal Remez,
  et~al.
\newblock Code llama: Open foundation models for code.
\newblock \emph{Arxiv Preprint Arxiv:2308.12950}, 2023.

\bibitem[Shah \& Wainwright(2018)Shah and Wainwright]{shah2018simple}
Nihar~B. Shah and Martin~J. Wainwright.
\newblock Simple, robust and optimal ranking from pairwise comparisons.
\newblock \emph{Journal of Machine Learning Research}, 18\penalty0
  (199):\penalty0 1--38, 2018.

\bibitem[Shi et~al.(2022)Shi, Fried, Ghazvininejad, Zettlemoyer, and
  Wang]{shi2022natural}
Freda Shi, Daniel Fried, Marjan Ghazvininejad, Luke Zettlemoyer, and Sida~I.
  Wang.
\newblock Natural language to code translation with execution.
\newblock In \emph{Proceedings of the 2022 Conference on Empirical Methods in
  Natural Language Processing}, pp.\  3533--3546. Association for Computational
  Linguistics, 2022.

\bibitem[Snell et~al.(2025)Snell, Lee, Xu, and Kumar]{snell2025scaling}
Charlie~Victor Snell, Jaehoon Lee, Kelvin Xu, and Aviral Kumar.
\newblock Scaling {LLM} test-time compute optimally can be more effective than
  scaling parameters for reasoning.
\newblock In \emph{The Thirteenth International Conference on Learning
  Representations}, 2025.

\bibitem[Sun et~al.(2024)Sun, Wan, Li, Zhang, Jin, Li, and Lyu]{sun2024sifting}
Zhihong Sun, Yao Wan, Jia Li, Hongyu Zhang, Zhi Jin, Ge~Li, and Chen Lyu.
\newblock Sifting through the chaff: On utilizing execution feedback for
  ranking the generated code candidates.
\newblock pp.\  229--241. Association for Computing Machinery, 2024.

\bibitem[Taherkhani et~al.(2026)Taherkhani, DaghighFarsoodeh, Chowdhury, Pham,
  and Hemmati]{convertest2026}
Hamed Taherkhani, Alireza DaghighFarsoodeh, Mohammad Chowdhury, Hung~Viet Pham,
  and Hadi Hemmati.
\newblock Consistency meets verification: Enhancing test generation quality in
  large language models without ground-truth solutions.
\newblock \emph{Arxiv Preprint Arxiv:2602.10522}, 2026.

\bibitem[To et~al.(2024)To, Huynh~Nguyen, and Bui]{to2024functional}
Hung~Quoc To, Minh Huynh~Nguyen, and Nghi D.~Q. Bui.
\newblock Functional overlap reranking for neural code generation.
\newblock In \emph{Findings of the Association for Computational Linguistics:
  Acl 2024}, pp.\  3686--3704. Association for Computational Linguistics, 2024.

\bibitem[Wang et~al.(2023{\natexlab{a}})Wang, Wei, Schuurmans, Le, Chi, Narang,
  Chowdhery, and Zhou]{wang2022self}
Xuezhi Wang, Jason Wei, Dale Schuurmans, Quoc~V Le, Ed~H. Chi, Sharan Narang,
  Aakanksha Chowdhery, and Denny Zhou.
\newblock Self-consistency improves chain of thought reasoning in language
  models.
\newblock In \emph{The Eleventh International Conference on Learning
  Representations}, 2023{\natexlab{a}}.

\bibitem[Wang et~al.(2025)Wang, Yang, Tian, Shen, and Wang]{wang2025cure}
Yinjie Wang, Ling Yang, Ye~Tian, Ke~Shen, and Mengdi Wang.
\newblock {CURE}: Co-evolving coders and unit testers via reinforcement
  learning.
\newblock In \emph{The Thirty-Ninth Annual Conference on Neural Information
  Processing Systems}, 2025.

\bibitem[Wang et~al.(2023{\natexlab{b}})Wang, Xu, Yang, He, Cao, and
  Huang]{wang2022optimizing}
Zitai Wang, Qianqian Xu, Zhiyong Yang, Yuan He, Xiaochun Cao, and Qingming
  Huang.
\newblock Optimizing partial area under the top-k curve: Theory and practice.
\newblock \emph{Ieee Transactions on Pattern Analysis and Machine
  Intelligence}, 45\penalty0 (4):\penalty0 5053--5069, 2023{\natexlab{b}}.

\bibitem[Wu et~al.(2025)Wu, Sun, Li, Welleck, and Yang]{wu2024inference}
Yangzhen Wu, Zhiqing Sun, Shanda Li, Sean Welleck, and Yiming Yang.
\newblock Inference scaling laws: An empirical analysis of compute-optimal
  inference for {LLM} problem-solving.
\newblock In \emph{The Thirteenth International Conference on Learning
  Representations}, 2025.

\bibitem[Xia et~al.(2025)Xia, Shen, Wang, Liu, Sun, Wu, Hu, and
  Xu]{xia2025leetcodedataset}
Yunhui Xia, Wei Shen, Yan Wang, Jason~Klein Liu, Huifeng Sun, Siyue Wu, Jian
  Hu, and Xiaolong Xu.
\newblock {LeetCodeDataset}: A temporal dataset for robust evaluation and
  efficient training of code {LLMs}.
\newblock \emph{arXiv preprint arXiv:2504.14655}, 2025.

\bibitem[Xie et~al.(2024)Xie, Liu, He, Li, and Zhou]{xie2023wsauc}
Zheng Xie, Yu~Liu, Hao-Yuan He, Ming Li, and Zhi-Hua Zhou.
\newblock Weakly supervised auc optimization: A unified partial auc approach.
\newblock \emph{Ieee Transactions on Pattern Analysis and Machine
  Intelligence}, 46\penalty0 (7):\penalty0 4780--4795, 2024.

\bibitem[Yang et~al.(2025)Yang, Kuang, Xia, and Zhao]{yang-etal-2025-llms}
Zheyuan Yang, Zexi Kuang, Xue Xia, and Yilun Zhao.
\newblock Can {LLM}s generate high-quality test cases for algorithm problems?
  {T}est{C}ase-eval: A systematic evaluation of fault coverage and exposure.
\newblock In \emph{Proceedings of the 63rd Annual Meeting of the Association
  for Computational Linguistics (volume 2: Short Papers)}, pp.\  1050--1063.
  Association for Computational Linguistics, 2025.

\bibitem[Zhang et~al.(2023{\natexlab{a}})Zhang, Wang, Xia, Wang, and
  Li]{zhang2023algo}
Kexun Zhang, Danqing Wang, Jingtao Xia, William~Yang Wang, and Lei Li.
\newblock {ALGO}: Synthesizing algorithmic programs with generated oracle
  verifiers.
\newblock In \emph{Thirty-Seventh Conference on Neural Information Processing
  Systems}, 2023{\natexlab{a}}.

\bibitem[Zhang et~al.(2023{\natexlab{b}})Zhang, Yu, Hashimoto, Lewis, Yih,
  Fried, and Wang]{zhang2023coder}
Tianyi Zhang, Tao Yu, Tatsunori Hashimoto, Mike Lewis, Wen-tau Yih, Daniel
  Fried, and Sida Wang.
\newblock Coder reviewer reranking for code generation.
\newblock In \emph{Proceedings of the 40th International Conference on Machine
  Learning}, pp.\  41832--41846. PMLR, 2023{\natexlab{b}}.

\bibitem[Zheng et~al.(2023)Zheng, Chiang, Sheng, Zhuang, Wu, Zhuang, Lin, Li,
  Li, Xing, Zhang, Gonzalez, and Stoica]{zheng2023judging}
Lianmin Zheng, Wei-Lin Chiang, Ying Sheng, Siyuan Zhuang, Zhanghao Wu, Yonghao
  Zhuang, Zi~Lin, Zhuohan Li, Dacheng Li, Eric Xing, Hao Zhang, Joseph~E.
  Gonzalez, and Ion Stoica.
\newblock Judging {LLM}-as-a-judge with {MT}-bench and chatbot arena.
\newblock In \emph{Thirty-Seventh Conference on Neural Information Processing
  Systems Datasets and Benchmarks Track}, 2023.

\end{thebibliography}
\end{document}